\documentclass{article}

\usepackage[utf8]{inputenc}
\usepackage[T1]{fontenc} 

\usepackage{times}
\usepackage{color}
\usepackage{fullpage}
\usepackage{graphicx}
\usepackage{epstopdf}
\usepackage{mathtools}
\usepackage{amsthm, amssymb}
\usepackage[colorlinks,linkcolor=blue,citecolor=blue,urlcolor=DarkBlue]{hyperref}
\usepackage[ruled, vlined, linesnumbered]{algorithm2e}
\usepackage{caption, subcaption}
\usepackage{enumitem}
\usepackage{comment}
\usepackage{placeins}
\usepackage{microtype}
\usepackage{makecell}
\usepackage{multirow}
\usepackage{booktabs}

\newcommand{\abs}[1]{\left| #1 \right|}

\newcommand{\E}{\mathcal{E}}

\usepackage{tcolorbox}
\tcbuselibrary{skins,breakable}
\tcbset{enhanced jigsaw}

\definecolor{DarkRed}{rgb}{0.5,0.1,0.1}
\definecolor{DarkBlue}{rgb}{0.1,0.1,0.5}
\definecolor{RURed}{rgb}{0.8,0.1,0.1}
\definecolor{ForestGreen}{rgb}{0.1333,0.5451,0.1333}
\definecolor{Red}{rgb}{0.9,0,0}

\usepackage{nameref}
\usepackage[noabbrev,nameinlink]{cleveref}
\crefname{property}{property}{Property}
\creflabelformat{property}{(#1)#2#3}
\crefname{equation}{eq}{Eq}
\creflabelformat{equation}{(#1)#2#3}

\newtheorem{theorem}{Theorem}
\newtheorem{lemma}{Lemma}[section]
\newtheorem{claim}[lemma]{Claim}

\newtheorem{proposition}{Proposition}
\newtheorem{fact}[lemma]{Fact}

\theoremstyle{definition}

\newtheorem{observation}[lemma]{Observation}

\newtheorem{remark}[lemma]{Remark}

\newtheorem{definition}{Definition}

\usepackage{mdframed}
\newtheorem{mdresult}{Result}
\newenvironment{result}{\begin{mdframed}[backgroundcolor=lightgray!40,topline=false,rightline=false,leftline=false,bottomline=false,innertopmargin=5pt]\begin{mdresult}}{\end{mdresult}\end{mdframed}}

\definecolor{RURed}{rgb}{0.8,0.1,0.1}

\newcommand{\tvd}[2]{\ensuremath{\norm{#1 - #2}_{\mathrm{tvd}}}}

\newcommand{\eps}{\ensuremath{\varepsilon}}

\newcommand{\bracket}[1]{\left[#1\right]}
\newcommand{\paren}[1]{\ensuremath{\left(#1\right)}\xspace}
\newcommand{\card}[1]{\left\vert{#1}\right\vert}

\newcommand{\norm}[1]{\ensuremath{\|#1\|}}

\newcommand{\expect}[1]{\Exp\bracket{#1}}
\newcommand{\expectR}[2]{\Exp_{#1}\bracket{#2}}

\newcommand{\set}[1]{\ensuremath{\left\{ #1 \right\}}}
\newcommand{\poly}{\mbox{\rm poly}}
\newcommand{\polylog}{\textnormal{polylog}\xspace}

\newcommand{\ALG}{\ensuremath{\mbox{\sc alg}}\xspace}

\DeclareMathOperator*{\Exp}{\ensuremath{{\mathbb{E}}}}
\DeclareMathOperator*{\Prob}{\ensuremath{\textnormal{Pr}}}
\renewcommand{\Pr}{\Prob}

\newenvironment{tbox}{\begin{tcolorbox}[
		enlarge top by=5pt,
		enlarge bottom by=5pt,
		 breakable,
		 boxsep=2pt,
                  left=5pt,
                  right=7pt,
                  top=10pt,
                  arc=0pt,
                  boxrule=1pt,toprule=1pt,
                  colback=white
                  ]
	}
{\end{tcolorbox}}

\newcommand{\kl}[2]{\ensuremath{\mathbb{D}(#1~||~#2)}}
\newcommand{\II}{\ensuremath{\mathbb{I}}}
\newcommand{\HH}{\ensuremath{\mathbb{H}}}

\newcommand{\mireal}[1][]{
  \ifx\relax#1\relax%
    \II(\mione \,; \mitwo)%
  \else%
    \II(\mione \,; \mitwo\mid #1)%
  \fi
}

\newcommand{\cD}{\mathcal{D}}

\newcommand{\Xyes}{\ensuremath{X_{\text{yes}}}\xspace}
\newcommand{\Xno}{\ensuremath{X_{\text{no}}}\xspace}
\newcommand{\Xhigh}{\ensuremath{X_{\frac{1}{2}+\alpha+\beta}}\xspace}
\newcommand{\Xflat}{\ensuremath{X_{\frac{1}{2}}}\xspace}

\newcommand{\cP}{\ensuremath{\mathcal{P}}\xspace}

\newcommand{\arm}{\ensuremath{\textnormal{arm}}\xspace}
\newcommand{\armstar}{\ensuremath{\arm^{\star}}\xspace}

\newcommand{\pitilde}{\ensuremath{\widetilde{\pi}}\xspace}
\newcommand{\itilde}{\ensuremath{\widetilde{i}}\xspace}

\newcommand{\bern}[1]{\ensuremath{\textnormal{Bern}(#1)}\xspace}

\newcommand{\Deltai}{\ensuremath{\Delta_{[i]}}\xspace}

\newcommand{\myqed}[1]{\let\qed\relax \hspace*{\fill} #1 \ensuremath{\square}}

\newcommand{\istar}{\ensuremath{i^{*}}\xspace}
\newcommand{\jstar}{\ensuremath{j^{*}}\xspace}

\newcommand{\etaib}[2]{\ensuremath{\eta^{(#1)}_{#2}}\xspace}

\newcommand{\smp}{\ensuremath{\mathtt{Smp^{\ALG}}}\xspace}

\newcommand{\F}{\mathcal{F}}
\newcommand{\bE}{\operatorname{\mathbb{E}}}

\newcommand{\cE}{\ensuremath{\mathcal{E}}\xspace}

\title{Nearly Tight Bounds for Exploration in Streaming Multi-armed Bandits with Known Optimality Gap}
\author{Nikolai Karpov \thanks{Indiana University. \texttt{email:~kimaska@gmail.com}}
\and 
Chen Wang\thanks{Rice University and Texas A\&M University. \texttt{email:~cwangjhw@tamu.edu}}
}
\date{} 

\begin{document}
\maketitle

\begin{abstract}
	We investigate the sample-memory-pass trade-offs for pure exploration in multi-pass streaming multi-armed bandits (MABs) with the \emph{a priori} knowledge of the optimality gap $\Delta_{[2]}$. 
	Here, and throughout, the optimality gap $\Delta_{[i]}$ is defined as the mean reward gap between the best and the $i$-th best arms. A recent line of results by Jin, Huang, Tang, and Xiao [ICML'21] and Assadi and Wang [COLT'24] have shown that if there is no known $\Delta_{[2]}$, a pass complexity of $\Theta(\log(1/\Delta_{[2]}))$ (up to $\log\log(1/\Delta_{[2]})$ terms) is necessary and sufficient to obtain the \emph{worst-case optimal} sample complexity of $O(n/\Delta^{2}_{[2]})$ with a single-arm memory. However, our understanding of multi-pass algorithms with known $\Delta_{[2]}$ is still limited. Here, the key open problem is how many passes are required to achieve the complexity, i.e., $O( \sum_{i=2}^{n}1/\Delta^2_{[i]})$ arm pulls, with a sublinear memory size.
	
	In this work, we show that the ``right answer'' for the question is $\Theta(\log{n})$ passes (up to $\log\log{n}$ terms).  We first present a lower bound, showing that any algorithm that finds the best arm with slightly sublinear memory -- a memory of $o({n}/{\polylog({n})})$ arms -- and $O(\sum_{i=2}^{n}{1}/{\Deltai^{2}}\cdot \log{(n)})$ arm pulls has to make $\Omega(\frac{\log{n}}{\log\log{n}})$ passes over the stream. We then show a nearly-matching algorithm that assuming the knowledge of $\Delta_{[2]}$, finds the best arm with $O( \sum_{i=2}^{n}1/\Delta^2_{[i]} \cdot \log{n})$ arm pulls and a \emph{single arm} memory.
\end{abstract}

\clearpage

\renewenvironment{quote}
{\list{}{\rightmargin=0.5cm \leftmargin=0.5cm}%
	\item\relax}
{\endlist}

\section{Introduction}
\label{sec:intro}
The pure exploration in multi-armed bandits (MABs) is one of the most well-studied problems in theoretical computer science (TCS) and machine learning (ML). The problem is formulated as follows: given $n$ arms with unknown sub-Gaussian reward distributions, find the best arm, defined as the arm with the highest mean reward, with a sufficiently high probability and a small number of arm pulls (sample complexity). Here, and throughout, the parameter $\Delta_{[i]}$ is defined as the mean reward gap between the best and the $i$-th best arms. The problem has been extensively studied in the literature (e.g., \cite{Even-Dar+02,MannorTs04,KalyanakrishnanSt10,KarninKS13,JamiesonMNB14,KCG16,CL16,AgarwalAAK17,ChenLQ17}, see~\cite{Slivkins19} for an excellent monograph), and its application has been found in numerous areas, e.g., experiment design \cite{Robbins52,villar2015multi,aziz2021multi}, recommendation systems \cite{silva2022multi}, search ranking \cite{AgarwalCEMPRRZ08,radlinski2008learning}, robot control \cite{KovalKPS15}, to name a few. 

The optimal sample complexity bound under the classical RAM setting has been established 
through a series of elegant works \cite{Even-Dar+02,MannorTs04,KarninKS13,JamiesonMNB14}. The pioneering work of \cite{Even-Dar+02} shows that if the value of $\Delta_{[2]}$ is known, there exists an algorithm that finds the best arm with high (constant) probability and a \emph{worst-case optimal} $O(n/\Delta^2_{[2]})$ arm pulls. Subsequently, the work of \cite{KarninKS13,JamiesonMNB14} improved the sample complexity to the \emph{nearly instance optimal} bound of $O(\sum_{i=2}^{n}1/\Delta^2_{[i]}\cdot \log\log(1/\Delta_{[i]}))$\footnote{We slightly overload the terminology to call both $O(\sum_{i=2}^{n}1/\Delta^2_{[i]}\cdot \log\log(1/\Delta_{[i]}))$ and $O(\sum_{i=2}^{n}1/\Delta^2_{[i]}\cdot \polylog{(n)})$ nearly instance optimal sample complexity, and denote (any of) them with $\tilde{O}(\sum_{i=2}^{n}1/\Delta^2_{[i]})$ when the context is clear.}, and their algorithms do \emph{not} require a known value of $\Delta_{[2]}$. On the lower bound side, \cite{MannorTs04} showed that a sample complexity of $\Omega(\sum_{i=2}^{n}1/\Delta^2_{[i]})$ is necessary to find the best arm with high constant probability, which completes the picture for classical pure exploration up to the doubly-logarithmic factor.\footnote{The seemingly artificial $\log\log(1/\Delta_{[i]})$ factor embodies some fundamental properties of the problem. See \cite{ChenL16,CLQ17} for more discussions.} 

Virtually all algorithms for pure exploration in the classical setting require all arms available in the memory for repeated visits. Aimed at modern large-scale applications, in which storing everything becomes impossible, an important direction to explore is MABs in the memory-constrained setting. To this end, \cite{AssadiW20} introduced the streaming MABs model, where the arms arrive one by one in a streaming manner, and the algorithm uses a limited memory to store, discard, and replace arms. The target for streaming MABs algorithms is to simultaneously minimize the sample complexity and the \emph{space complexity} -- the maximum number of arms stored at any point. \cite{AssadiW20} showed that if the value of $\Delta_{[2]}$ is provided, there exists a \emph{single-pass} algorithm that finds the best arm with high constant probability, the worst-case optimal $O(n/\Delta^2_{[2]})$ sample complexity, and a \emph{single-arm} memory. The key conceptual message of \cite{AssadiW20} is that in the regime where $\Delta_{[2]}$ is known and the target is the worst-case optimal sample complexity, there is no sample-space trade-off in this setting.

The results of \cite{AssadiW20} have led to considerable interest in understanding the power and limitations of the streaming MABs model \cite{MaitiPK21,JinH0X21,AWneurips22,AgarwalKP22,Wang23Regret,LZWL23,HYZ25SODA}. In particular, since \cite{AssadiW20} only deals with the setting when $\Delta_{[2]}$ is given and the worst-case optimal bound, a natural question is to ask what happens if $\Delta_{[2]}$ is unknown, or if the target becomes the (nearly) instance optimal bound instead. Unfortunately, in these settings, the optimistic message in \cite{AssadiW20} no longer holds: \cite{AWneurips22} showed that in the single-pass setting, if the value of $\Delta_{[2]}$ is not given a priori, then the sample complexity is unbounded unless the algorithm has $\Omega(n)$-arm memory. Furthermore, even if the value of $\Delta_{[2]}$ is known, there is a sample complexity lower bound of $\Omega(n/\Delta^2_{[2]})$ for any algorithm with $o(n)$-arm memory in the single-pass streaming setting. These results assert that multiple passes over the stream are necessary if we want any streaming algorithms with sublinear memory under the new settings.

It turns out that allowing multiple passes does lead to improved bounds. Concretely, \cite{JinH0X21} shows that in $O(\log(1/\Delta_{[2]}))$ passes, it is possible to find the best arm with a single-arm memory and the near-instance optimal $O(\sum_{i=2}^{n}1/\Delta^2_{[i]}\cdot \log\log(1/\Delta_{[i]}))$ arm pulls. Furthermore, the algorithm does \emph{not} require a known quantity of $\Delta_{[2]}$. Recently, \cite{AW23BestArm} proved that for a streaming algorithm with $o(n)$-arm memory to achieve even the worst-case optimal bound $O(n/\Delta^2_{[2]})$ bound with the stream alone, a number of $\Omega\left(\frac{\log(1/\Delta_{[2]})}{\log\log(1/\Delta_{[2]})}\right)$ passes is necessary. As such, we already established a good understanding of the pass-sample-space trade-off for multi-pass algorithms without $\Delta_{[2]}$ value.

The final missing piece to complete the theoretical picture of multi-pass streaming MABS is to understand the case when $\Delta_{[2]}$ is provided \emph{and} the target is the instance-optimal sample complexity.
We remark this question is not trivial: in the adversarial instance distribution of \cite{AW23BestArm}, if $\Delta_{[2]}$ is provided, we can uniquely determine the realization of instances in their distribution. As such, it is possible that if $\Delta_{[2]}$ is known, there are algorithms with better efficiency. 
This open question can be formally presented as follows.

\begin{quote}
	\centering
	\it If the value of $\Delta_{[2]}$ is known \emph{a priori}, what is the optimal number of passes for a streaming algorithm with $o(n)$ arm memory to find the best arm with the (nearly) instance optimal sample complexity?
\end{quote}

We now provide some additional discussions to better motivate the investigation. The importance of the open question could be summarized as follows.
\begin{itemize}
\item The question is important for the theoretical foundations of \emph{online learning}. The streaming MABs model has been widely regarded as an important models for modern large-scale online learning \cite{MaitiPK21,JinH0X21,AWneurips22,AgarwalKP22,Wang23Regret,LZWL23,AW23BestArm}. Since the knowledge of $\Delta_{[2]}$ is frequently assumed in the literature, as evidenced by works such as \cite{Even-Dar+02,AssadiW20}, the motivating question is an important missing piece to be answered for multi-pass pure exploration. As such, the primary motivation for the investigation is to complete the theoretical picture for the streaming MABs model.
\item Algorithms with better sample and pass efficiency could lead to \emph{direct application} in various tasks. For instance, if we want to find the best seller in large-scale online retailers, we could query the data from data centers in a streaming fashion, and run the algorithm using the local RAM. Here, we could \emph{estimate} the value of $\Delta_{[2]}$ from historical data (we only require a lower bound of such estimations for the algorithms to work, see \Cref{obs:delta-lower-bound} for more details). And since the products often have very different scales of transcactions, which implies that $\frac{n}{\Delta^2_{[2]}}\gg \sum_{i=2}^{n}\frac{1}{\Delta^2_{[i]}}$, our algorithm is much more efficient than the algorithm of \cite{AssadiW20,JinH0X21}.
\end{itemize}

\subsection{Our Contributions}
\label{sec:contribution}
Our main contribution is the answer to open question: we provide nearly-matching (up to exponentially smaller terms) upper and lower bounds for streaming algorithms with $o(n)$-arm memory to find the best arm with a (nearly) instance optimal sample complexity. We first present our lower bound result as follows.

\begin{result}[Lower bound, informal of \Cref{thm:lb-main}]
	\label{rst:main-lb}
	\vspace{-3pt}
	Any streaming algorithm that given $n$ arms in  a stream and a known value of $\Delta_{[2]}$, finds the best arm with high constant probability, $\tilde{O}(\sum_{i=2}^{n}\frac{1}{\Delta^2_{[i]}})$\footnote{In \Cref{rst:main-lb} and \Cref{rst:main-ub}, we use $\tilde{O}(\cdot)$ to hide $\polylog(n)$ terms.} arm pulls, and $o(\frac{n}{\polylog(n)})$ arm memory has to use $\Omega(\frac{\log(n)}{\log\log(n)})$ passes over the stream.
\end{result}

\Cref{rst:main-lb} shares a similar form of the \cite{AW23BestArm}, albeit we are able to substitute $\log(1/\Delta_{[i]})$ terms with $\log{n}$ terms. Before our results, the only known result for streaming MABs lower bounds with instance-optimal sample complexity is th result of \cite{AWneurips22}, which only works for a \emph{single} pass. Therefore, \Cref{rst:main-lb} marks a significant improvement in the pass complexity of the problem.

A natural question to follow up \Cref{rst:main-lb} is to answer whether the lower is optimal, i.e., is it possible to design an algorithm with a sample and pass complexity that matches the lower bound in \Cref{rst:main-lb}. We answer this question in the affirmative by showing an algorithm as in \Cref{rst:main-ub}.

\begin{result}[Upper bound, informal of \Cref{thm:basic}]
	\label{rst:main-ub}
	\vspace{-3pt}
	There exists a streaming algorithm that given $n$ arms in a stream and the value of $\Delta_{[2]}$, finds the best arm with high constant probability with $\tilde{O}(\sum_{i=2}^{n}\frac{1}{\Delta^2_{[i]}})$ arm pulls, $O(\log(n))$ passes over the stream, and a memory of a single arm.
\end{result}
In fact, the guarantee of our algorithm extends beyond just $O(\log(n))$ passes. We can always keep using only a single arm memory, and set $P$ passes to achieve $O(\sum_{i = 2}^n \frac{1}{\Delta^2_{[i]}} \cdot n^{2/P} \cdot \log \left({n P}\right))$ sample complexity. The sample-pass trade-off is optimized by taking $P=O(\log(n))$. The sample complexity bound of $O(\sum_{i=2}^{n}\frac{1}{\Delta^2_{[i]}}\cdot \log(n))$ we obtain here is slightly different from the classical near-instance optimal bound of $O(\sum_{i=2}^{n}\frac{1}{\Delta^2_{[i]}}\cdot \log\log{(1/\Delta_{[i]})})$. We remark that the $\log(n)$ multiplicative factor does \emph{not} render the bound trivial -- see \Cref{rmk:logn-overhead} for more details.

\Cref{rst:main-ub} suggests that algorithms with known $\Delta_{[2]}$ values could have much better pass efficiency. Note that the lower bound in \cite{AW23BestArm} holds with $\Delta_{[2]}$ values as small as $2^{n^{O(1)}}$, which means we may be forced to take $\poly(n)$ passes if $\Delta_{[2]}$ is unknown. In contrast, in the case when $\Delta_{[2]}$ is known, it becomes possible to find the best arm in $\log(n)$ passes, which is much smaller and reasonable in practice.

To make it easier for the readers to understand the context and contributions of our results, we provide \Cref{tab:results-comparison} that illustrates the comparison between the existing results and ours.

\begin{table}[!h]
	\centering
	\begin{tabular}{@{}c|c|c|c|c@{}}
		\toprule
		Pass & $\Delta_{[2]}$ is given & Target Sample Complexity & Memory & Remark and Reference\\ \midrule
		1 & Yes & $O(\frac{n}{\Delta^2_{[2]}})$  & Single arm & Upper bound, \cite{AssadiW20}  \\
		1 & No & $O(\frac{n}{\Delta^2_{[2]}})$  & $\Omega(n)$ arms & Lower bound, \cite{AWneurips22} \\
		1 & Yes & $\tilde{O}(\sum_{i=2}^{n}\frac{1}{\Delta^2_{[i]}})$  &  $\Omega(n)$ arms & Lower bound, \cite{AWneurips22} \\
		$O(\log(\frac{1}{\Delta_{[2]}}))$ & No & $\tilde{O}(\sum_{i=2}^{n}\frac{1}{\Delta^2_{[i]}})$ & Single arm & Upper bound, \cite{JinH0X21} \\
		$O(\frac{\log(\frac{1}{\Delta_{[2]}})}{\log\log(\frac{1}{\Delta_{[2]}})})$ & No & $O(\frac{n}{\Delta^2_{[2]}})$ & $\Omega(n/\polylog(\frac{1}{\Delta_{[2]}}))$ arms & Lower bound, \cite{AW23BestArm} \\
		$O(\frac{\log(n)}{\log\log(n)})$ & Yes & $\tilde{O}(\sum_{i=2}^{n}\frac{1}{\Delta^2_{[i]}})$ & $\Omega(n/\polylog(n))$ arms & Lower bound, \Cref{rst:main-lb} \\
		$O(\log{n})$ & Yes & $\tilde{O}(\sum_{i=2}^{n}\frac{1}{\Delta^2_{[i]}})$ & Single arm & Upper bound, \Cref{rst:main-ub} \\
		 \bottomrule
	\end{tabular}
	\caption{\label{tab:results-comparison}Summary of the previous results and our new results. To present the sample-memory-pass trade-offs, we set upper bounds on the number of passes and sample complexity, and show the memory with both upper and lower bounds. 
	}
\end{table}




\Cref{rst:main-lb} and \Cref{rst:main-ub} demonstrate a sharp memory-pass trade-off for streaming algorithms to find the best arm with a near instance optimal sample complexity: with $O(\log(n))$ passes, we can obtain an algorithm with a memory of a single arm. However, if we decrease the number of passes slightly to $o(\log(n)/\log\log(n))$, no streaming algorithm will be able to achieve the sample complexity and success probability guarantee unless it uses almost $n$-arm memory. We note that this kind of dichotomy frequently arises in the streaming MABs literature (\cite{AssadiW20,AWneurips22,AgarwalKP22,AW23BestArm}), and we obtain a similar phenomenon in the multi-pass setting with a known $\Delta_{[2]}$ as well (see~\Cref{tab:results-comparison} for some examples).

Finally, we observe that by simply running our streaming algorithm in the offline setting, we obtain an offline algorithm that finds the best arm with high constant probability and $O(\sum_{i=2}^{n}\frac{1}{\Delta^2_{[i]}}\cdot \log(n))$ arm pulls. This observation, while straightforward, has the following implication on the role of $\Delta_{[2]}$ in the pure exploration MABs problem. The classical \emph{nearly} instance optimal sample complexity bound by \cite{KarninKS13,JamiesonMNB14} is $O(\sum_{i=2}^{n}\frac{1}{\Delta^2_{[i]}}\cdot \log\log(\frac{1}{\Delta_{[i]}}))$. In the two-arm scenario, \cite{JamiesonMNB14} also proved that $\Omega(\frac{1}{\Delta^2_{[2]}}\cdot \log\log(\frac{1}{\Delta_{[2]}}))$ arm pulls are \emph{necessary} (when the value of $\Delta_{[2]}$ is unknown). \cite{ChenLi15} further improved the upper bound to $O(\sum_{i=2}^{n}\frac{1}{\Delta^2_{[i]}}\cdot \log\log(\min\{n, \frac{1}{\Delta_{[i]}}\}) + \frac{1}{\Delta^2_{[2]}}\cdot \log\log(\frac{1}{\Delta_{[2]}}))$, and the discussion went deeper with later results by \cite{ChenL16,CLQ17}, in which they proposed the `gap-entropy conjecture' for the optimal sample complexity bound (also for the case of unknown $\Delta_{[2]}$, see \cite{ChenL16} for details). In contrast, our observation shows that if $\Delta_{[2]}$ is provided, we can instead get a multiplicative term that is independent of all $\Delta_{[i]}$ values in the logarithmic term. 

\paragraph{Experiments.} To validate the performance of our algorithm, we conduct experiments on multiple types of streaming MABs instances. We compare our algorithm with two benchmarks: $i).$ the single-pass algorithm by \cite{AssadiW20}, which enjoys the ultimate pass efficiency of a single pass, but only guarantees the worst-case optimal $\Theta(\frac{n}{\Delta^2_{[2]}})$ sample complexity; and $ii).$ the  $O(\log(1/\Delta_{[2]}))$-pass algorithm by \cite{JinH0X21}, which has the advantage of not requiring the knowledge of $\Delta_{[2]}$, but has to use more passes. Our result shows that in multiple setting, our algorithm consistently enjoys the best sample efficiency. Furthermore, comparing ot the algorithm of \cite{JinH0X21}, our algorithm uses significantly less passes over the stream. The results of our experiments are presented in \Cref{sec:experiment}.

\subsection{Our Techniques}
\label{sec:tech-overview}

\paragraph{Lower Bound.} Proving lower bounds for multi-pass streaming MABs typically involves intricate techniques to carefully manage memory, samples, and the information revealed over time. To navigate these technical challenges, we draw inspiration from recent work by \cite{AW23BestArm}, which established lower bounds for multi-pass MABs without prior knowledge of $\Delta_{[2]}$. The lower bound construction devised by \cite{AW23BestArm} employs a 'batched' approach to instance distributions. Roughly, it divides the arms into $B+1$ batches for a $P$-pass algorithm, where $P\leq B$. Within each batch $b$, all but one arm consistently yield a mean reward of $\frac{1}{2}$, while the remaining 'special arm' offers stochastic mean rewards --either $\frac{1}{2}$ or $\frac{1}{2}+\alpha_{b}$ for some $\alpha_{b}>0$-- placed uniformly at random among the arms within batch $b$. Importantly, the later-arriving batches \emph{might} have higher mean rewards. This construction allows us to argue that to maintain optimal sample complexity, the algorithm must always check whether a batch contains an arm with a reward exceeding $1/2$ from the \emph{latest} batch that has not been checked, which forms a lower bound.

	The above sketches the \emph{intuition} of \cite{AW23BestArm}, and the actual proof is considerably more involved. Despite the very technical analysis, we observe that we could actually extract a framework from \cite{AW23BestArm} to capture the memory-sample trade-off for algorithms on batched instances. In particular, to establish such trade-offs, we only need $i).$ a sample complexity lower bound for the streaming algorithm to `trap' the special arm from a batch; $ii).$ a sample lower bound for the streaming algorithm to `learn' the distribution for each batch; and $iii).$ a sufficiently high gap between the sample complexity for different batches. We remark that these aspects are not explicitly written in \cite{AW23BestArm}, and forming the framework from key observations is one of our technical contributions.
	
	With this novel technical framework, a natural idea to prove lower bounds for the instance-sensitive $O(\sum_{i=2}^{n}1/\Delta^2_{[i]} \cdot \log(n))$ sample complexity is to construct a batched instance distribution that $a).$ satisfies the properties as required by the framework, and $b).$ varies the quantity of $O(\sum_{i=2}^{n}1/\Delta^2_{[i]} \cdot \log(n))$ with different realizations.
	To this end, we note that we could \emph{not} directly use the construction of \cite{AW23BestArm} in our setting. 
	An obvious problem here is that since the values of $\alpha_{b}$ change across batches, the information of $\Delta_{[2]}$ alone can help uniquely identify the realization of the distribution, which renders the distribution not hard. The idea to resolve this issue is to have \emph{two} special arms, whose mean rewards are with $\frac{1}{2}+\chi_{b}$ and $\frac{1}{2}+\chi_{b}+\gamma$. Moreover, we make $\chi_{b}$ to be much larger than $\gamma$ for any $b$, but progressively smaller by $1/\poly(B)$ factor across the batches. As such, we always have $\Delta_{[2]}=\gamma$, which means the value of $\Delta_{[2]}$ no longer reveals any information about the instance realization.
	
	The final missing piece is to show how `hard' the new construction is for streaming algorithms. Unfortunately, due to the introduction of an extra special arm, the standard technical tools developed from \cite{AgarwalKP22,AWneurips22,AW23BestArm} no longer work. To overcome the issue, we use the information-theoretic tools to develop several new results for \emph{double-armed bandits}, and apply the `direct-sum' idea in \cite{AWneurips22,AW23BestArm} to obtain new sample complexity bounds for batches with two special arms. To the best of our knowledge, this is the first lower bound that studies the sample complexity in the \emph{double-armed} setting, which could be of independent interest. Finally, by taking $B=\Theta(\log(n)/\log\log(n))$, we can simultaneously ensure the value gap between $\chi_{b}$ and $\gamma$ and the sample complexity gap between batches, which gives the desired result.
	
	\paragraph{Upper Bound.} Our algorithms work with the elimination-based approach extensively studied in the MABs literature~\cite{HKK+13,KZZ20}. 
The state-of-the-art multi-pass streaming algorithm by \cite{JinH0X21} is a streaming adaptation of the classical elimination algorithm \cite{KarninKS13}. 
	The algorithm by \cite{JinH0X21} requires \(O\left(\log \left(\frac{1}{\Delta_{[2]}}\right)\right)\) passes, and the main idea is to ``binary search'' the ``correct'' gap parameters.
	Concretely, at the beginning of the $p$-th pass, the algorithm fixes an elimination gap \(\epsilon_p = O(2^{-p})\), the goal of the algorithm at the end of the \(p\)-pass is to eliminate all arms \(i\) such that \(\Delta_i > \epsilon_p\)\footnote{We use the notation $\Delta_{i}$ (without the brackets on $i$) to denote the gap between the best arm and $\arm_{i}$. See \Cref{subsec:notation} for more clarifications.} with roughly \(O({1}/{\epsilon^2_p})\) arm pulls on arm $i$. 
	After \(O\left(\log \left(\frac{1}{\Delta_{[2]}}\right)\right)\) passes, the value of the elimination gap becomes smaller than \(\Delta_{[2]}\), and all arms except the best arm can be safely eliminated. 
	
	To make the number of passes independent of $\Delta_{[2]}$, our key observation is that the ``binary search'' in the elimination procedure can be made more efficient with geometric series from $n\Delta_{[2]}$ to $\Delta_{[2]}$. Concretely, instead of initiating the elimination process of arms with constant reward gap, we choose the sequence of elimination gaps in the form of \(\epsilon_p = {\Delta_{[2]}}\cdot n^{1-p/P}\) for the $p$-th pass. Here, \(P\) is the total number of passes for the algorithm. On a very high level, it is easy to observe that after \(P\) passes, the elimination gaps become smaller than \(\Delta_{[2]}\), and all arms except the best arm can be safely eliminated (with high probability). 
	The observation generalizes to any $\arm_{i}$ with gap parameter \(\Delta_i\): when \(\epsilon_p\) become smaller than \(\Delta_i\), a sup-optimal arm $\arm_{i}$ is eliminated with high probability. The correctness of the algorithm thus follows from the high probability event that only the best arm will remain after \(P\) passes.
	

	For the analysis of sample complexity,
	we proceed by categorizing the set of arms into two parts with large gaps \(\Delta_i > n \Delta_{[2]}\) and small gaps \(\Delta_i \le n\Delta_{[2]}\). 
	The analysis for a small gaps group controls that the number of pulls assigned to each arm is at most \(O(\frac{n^{1/P}}{\Delta_i^2})\). Similar to the analysis of \cite{KarninKS13,JinH0X21}, the key observation here is that after the number of pulls used on $\arm$ $i$ becomes larger than \(\frac{1}{\Delta^2_i}\), we can eliminate such a suboptimal arm, 
	which implies that we spend at most \(\frac{n^{2/P}}{\Delta^2_i}\) pulls for arm \(i\) with high probability. 
	On the other hand, for the arms with large gaps, we observe that the sample complexity of \emph{all} arms with gaps more than $n\Delta_{[2]}$ is dominated by a single largest term \(\frac{1}{\Delta^2_{[2]}}\), which leads to the desired sample complexity bound.

\section{Preliminary}\label{sec:prelim}
We introduce the notation and formal description of the streaming MABs model in this section. We provide more technical preliminaries in \Cref{sec:standatd-tech-tools}.

\subsection{Notation}
\label{subsec:notation}
Throughout, we use $n$ to denote the number of arms. We use $i$ to denote the indices of the arms, and we have the set of indices as \(I\) (which is a permutation of $[n]$). We let \(\mu_i\) be the mean of the $i$-th arm; furthermore, we denote the index of the best arm as \(\star:= \arg\,\max_{i \in I} \mu_i \). As such, the best arm is denoted as $\armstar$, and the mean of the best arm is \(\mu_{\star}\). The reward gap between the best and the \(i\)-th arm is equal to \(\Delta_i := \mu_{\star} - \mu_i\). We also use the ordered sequence of gaps \(\Delta_{[2]} \le \Delta_{[3]} \le \dotsc \le \Delta_{[n]}\), i.e., $\Delta_{[i]}$ is the reward gap between the best and the $i$-th \emph{best} arm.

We frequently deal with Bernoulli random variables. For convenience, we use $\bern{\mu}$ to denote a Bernoulli distribution with mean $\mu$ (i.e., the probability to sample $1$ is $\mu$). When random variables and their realizations are presented side by side, as a convention, we use upper cases (e.g., $\Pi$) to denote the random variables and lower cases (e.g., $\pi$) to denote the realizations.

\subsection{The Streaming Multi-armed Bandits Model}
\label{subsec:model}
We now formally introduce the streaming MABs model as follows. There is a collection of $n$ arms, denoted as $\{\arm_{i}\}_{i=1}^{n}$, and their reward distributions are characterized by $\{\bern{\mu_{i}}\}_{i=1}^{n}$\footnote{Our upper and lower bounds apply to all sub-gaussian distributions, see \Cref{rmk:general-sub-gaussian} for discussions}. As the name suggested, the arms arrive one after another in an \emph{arbitrary and fixed} order (a permutation over $[n]$). Here, an \emph{arbitrary} order means the arrival order of the arms is selected by an adversary and can be in the worst case, and a \emph{fixed} order means that the order of arrival for the arms is the same across different passes.

A multi-pass streaming algorithm in the streaming MABs setting is defined as an algorithm that maintains a memory $M$, which is a set of arms, and a transcript $\pi$, which encodes the statistics of all past arm pulls. Each record in $\pi$ is a tuple that specifies the identity of the pulled arm, the result, and the pass index when the sample happened.

At any point, the streaming algorithm is allowed to make an arbitrary number of arm pulls on the \emph{arriving arm} and the arms \emph{stored in the memory}. The algorithm is allowed to make the following updates to the memory $M$:
\begin{enumerate}
	\item Adding the arriving arm to $M$.
	\item Discard the arriving arm, and continue to the next arriving arm.
	\item Discard arm(s) from the memory $M$.
\end{enumerate}
We define the \emph{sample complexity} as the number of arm pulls the streaming algorithm ever uses, and the \emph{space complexity (memory complexity)} as the maximum number of arms stored at any point (the maximum size of $M$). 
The common assumption in the literature \cite{AssadiW20,MaitiPK21,JinH0X21,AgarwalKP22,AWneurips22} allows the algorithm to write an arbitrary number of statistics for free, i.e., do not charge costs for the size of $\pi$ and any other stored information. 

\paragraph{The model with bounded statistics.} As we have mentioned, we provide an additional result to optimize the size of statistics in \Cref{sec:ub-stat-efficient}. To this end, we need to define the memory efficiency of the statistics. Let $\pitilde$ be a \emph{stored transcript} defined as follows: for each tuple that contains the arm pull and the result, the algorithm decides whether to write the tuple to $\pitilde$. In this setting, the algorithm is \emph{not} allowed to revisit all the past arm pulls, but only the ones that are stored. We define the memory complexity for the statistics as the maximum size of $\pitilde$ plus the maximum bits for the auxiliary information stored at any point.


\begin{remark}
\label{rmk:general-sub-gaussian}
We work with arms of Bernoulli distribution for both our upper and lower bounds for the convenience of presentation. We remark that our results apply to MABs with general (discrete) sub-Gaussian distributions. Concretely, we can assume w.log. that the supports are on $[0,1]$ by rescaling. For our upper bound result, the only place we used the property for Bernoulli distributions is when using the Chernoff-Hoeffding inequality (\Cref{lem:chernoff}), which holds for all sub-Gaussian distributions.
For the lower bound, we note that since the Bernoulli distribution does belong to the sub-Gaussian family, proving lower bounds on Bernoulli arms automatically implies lower bounds for sub-Gaussian arms.
\end{remark}

\section{Technical Lemmas for the Lower Bound}
\label{sec:tech-lemma}
In this section, we present several technical lemmas en route to our main lower bound result. In particular, we show the following results in this section.

\begin{enumerate}[label=\alph*).]
	\item A lower bound on the necessary number of arm pulls for an algorithm with sublinear memory to \emph{store} an arm with high mean reward (while possibly without knowing the identity of the arm).
	\item A lower bound on the necessary number of arm pulls for an algorithm to gain ``knowledge'' about the underlying \emph{distribution} of the MABs instance. 
	\item An observation of the framework from \cite{AW23BestArm} as a general sample-memory-pass trade-off for batched instances.
\end{enumerate}

We note that a variant of the first two lower bounds on instances with a single arm with high mean reward is first proved in \cite{AW23BestArm}. However, the subtle difference in the construction (as it will be evident in \Cref{sec:lb-main}) requires lower bounds to work with \emph{two} arms with high mean rewards. This, in particular, requires a careful handling of properties on Double-armed Bandits (DABs), which we prove in \Cref{lem:arm-identify} and \Cref{lem:arm-learn}.

\paragraph{Additional notation.} We introduce several additional notation used in a self-contained manner in the lower bound proof. Unless specified otherwise, we use $\ALG$
to denote a streaming algorithm, and $\smp$ is the random variable for the sample complexity of $\ALG$. As we introduced in \Cref{sec:standatd-tech-tools}, for two random variables $X$ and $Y$, we use $\tvd{X}{Y}$ to denote their total variation distance, $\kl{X}{Y}$ for the KL-divergence, and $\II(X;Y)$ for the mutual information. We also use $X\mid Y=y$ to denote the random variable for $X$ \emph{conditioning} on the realization of $Y=y$. Finally, for the conciseness of notation, we slightly abuse the notation to use $\kl{X|Z}{Y|Z}$ as a short-hand notation for $\Exp_{z \sim Z} \kl{X \mid Z=z}{Y \mid Z=z}$.

\subsection{Lower Bounds on the Sample Complexity of Double-armed Bandits}
\label{subsec:two-arm-lb}

We start with proving the necessary number of arm pulls to distinguish an instance of \emph{two} arms that are $(i).$ either both with reward $1/2$ or $(ii).$ one arm with mean reward $1/2+\alpha$ and the other with $1/2+\alpha+\beta$. The problem is in the same spirit as the single-arm distinguishment problem in \cite{AWneurips22,AW23BestArm}, but we are unaware of any previous result for the exact version we are using. 

Our first lemma shows that if an instance is sampled from a two-arm version of ``good arms'' and ``bad arms'', the algorithm will not be able to distinguish the cases if the number of arm pulls is small.


\begin{lemma}
	\label{lem:arm-identify}
	Consider two arms with a Bernoulli reward distribution whose mean is parameterized as follows.
	\begin{itemize}
		\item With probability $\rho$, the \emph{Yes case}, where
		\begin{enumerate}[label=\roman*).]
			\item $\arm_{1}$ is with reward $\frac{1}{2}+\alpha$;
			\item $\arm_{2}$ is with reward $\frac{1}{2}+\alpha+\beta$.
		\end{enumerate}
		\item With probability $1-\rho$, the \emph{No case}, where both $\arm_{1}$ and $\arm_{2}$ are with mean rewards of $\frac{1}{2}$;
	\end{itemize}
	where $\rho\in (0,\frac{1}{2}]$ is the probability for the reward to be more than $\frac{1}{2}$, and $\alpha, \beta >0$ satisfy $\alpha+\beta<\frac{1}{2}$. Any algorithm to determine the reward of the arms with a success probability of at least $(1-\rho+\eps)$ has to use $\frac{1}{4}\cdot \frac{\eps^2}{\rho^2 (\alpha+\beta)^{2}}$ arm pulls.
\end{lemma}
\begin{proof}
	We define $\Xyes=(\Xyes^{1}, \Xyes^{2}, \cdots, \Xyes^{m})$ as the random variable for taking $m$ samples from the Yes case. 
	Similarly, we define $\Xno=(\Xno^{1}, \Xno^{2}, \cdots, \Xno^{m})$ as the random variable for taking $m$ samples from the No case. 
	Furthermore, we also define random variables for ``dummy'' arm pulls: we define $\Xhigh$ as the random variable for taking a sample on an arm with reward $\frac{1}{2}+\alpha+\beta$, and $\Xflat$ as the random variable for taking samples on an arm with reward $\frac{1}{2}$.
	We first observe that for any $i\in [m]$, there is 
	\begin{align*}
		\kl{\Xyes^{i}}{\Xno^{i}}\leq \kl{\Xhigh^{i}}{\Xflat^{i}} =  \kl{\bern{\frac{1}{2}+\alpha+\beta}}{\bern{\frac{1}{2}}}.
	\end{align*}
	To see this, note that the algorithm is allowed to take a sample from either of the arms; however, the case to maximize the KL-divergence is for the algorithm to compare the empirical rewards from a $\bern{\frac{1}{2}+\alpha+\beta}$ arm and a $\bern{\frac{1}{2}}$ arm, which establishes the upper bound.
	
	We can in fact extend the above observation to \emph{conditional} KL-divergence. In particular, we have
	
	\begin{claim}
		\label{clm:cross-trial-ub}
		For any $i\in [m]$, there is
		\begin{align*}
			\kl{\Xyes^{i}\mid (\Xyes^{i+1},\cdots, \Xyes^{m})}{\Xno^{i}\mid (\Xno^{i+1},\cdots, \Xno^{m})} \leq \kl{\Xhigh}{\Xflat}.
		\end{align*}
	\end{claim}
	\begin{proof}
		Intuitively, the dependence between the results of arm pulls is only on the \emph{choice} of arms; once an arm is picked, the results are independent across different arm pulls. Our proof is a formalization of the above intuition. Define $\Xyes^{i, \arm_j}$ and $\Xno^{i, \arm_j}$ as the random variable for the algorithm to pull the $\arm_{j}$ ($j\in \{1,2\}$) on the $i$-th trial under the Yes and No cases, respectively. Furthermore, define $J$ as the random variable for the choice of arm by the algorithm. Note that we have $\Xyes^{i, \arm_j} = \Xyes^{i}\mid J=j$. 
		For any $i\in [m]$ and $j\in\{1,2\}$, there is
		\begin{align*}
			& \kl{\Xyes^{i}\mid (\Xyes^{i+1},\cdots, \Xyes^{m})}{\Xno^{i}\mid (\Xno^{i+1},\cdots, \Xno^{m})} \\
			& \leq \kl{\Xyes^{i}\mid (\Xyes^{i+1},\cdots, \Xyes^{m}, J)}{\Xno^{i}\mid (\Xno^{i+1},\cdots, \Xno^{m}, J)}  \tag{extra conditioning can only increase KL-divergence}\\
			& = \kl{\Xyes^{i}\mid J}{\Xno^{i}\mid J} \tag{indenpendence between sampling from bernoulli distributions}\\
			& \leq \kl{\Xhigh}{\Xflat},
		\end{align*}
		as desired. \myqed{\Cref{clm:cross-trial-ub}}
	\end{proof}
	
	We now use \Cref{clm:cross-trial-ub} to prove \Cref{lem:arm-identify}. By the standard calculation of the KL-divergence of Bernoulli random variables (\Cref{clm:bernoulli-KL}), we accordingly have
	\begin{align*}
		\kl{\Xhigh}{\Xflat} \leq 8 \cdot (\alpha+\beta)^2
	\end{align*}
	for any sample index of $i$. As such, we can bound the KL-divergence of the distributions with all samples as follows.
	\begin{align*}
		\kl{\Xyes}{\Xno} &= \sum_{i=1}^{m} \kl{\Xyes^{i}\mid (\Xyes^{i+1},\cdots, \Xyes^{m})}{\Xno^{i}\mid (\Xno^{i+1},\cdots, \Xno^{m})} \tag{by Chain rule}\\
		&\leq \sum_{i=1}^{m} \kl{\Xhigh}{\Xflat} \tag{by \Cref{clm:cross-trial-ub}}\\
		&= 8m\cdot (\alpha+\beta)^2.
	\end{align*}
	
	Therefore, by Pinsker's inequality \Cref{fact:pinsker}, we have 
	\begin{align*}
		\tvd{\Xyes}{\Xno} & \leq \sqrt{\frac{1}{2}\cdot \kl{\Xyes}{\Xno}}\\
		& \leq 2 (\alpha+\beta) \cdot \sqrt{m}.
	\end{align*}
	On the other hand, by \Cref{fact:distinguish-tvd}, we know that to distinguish the cases by a sample from the distribution with probability at least $1-\rho-\eps$, there has to be $\tvd{\Xyes}{\Xno}\geq \frac{\eps}{\rho}$. As such, we get a lower bound of
	\begin{align*}
		m \geq \frac{1}{4}\cdot \frac{\eps^2}{\rho^2 (\alpha+\beta)^2},	
	\end{align*}
	as desired.
\end{proof}

We now move to the second result for double-armed bandits, which shows that if the number of arm pulls is small, then the ``knowledge'' 
of the algorithm cannot change the original distribution by too much. More formally, we prove that with a limited number of arm pulls, from the algorithm's perspective, the probability for which case the instance is from remains close to the original distribution.
\begin{lemma}
	\label{lem:arm-learn}
	Let $\alpha, \beta \in (0,\frac16)$, $\beta\leq \alpha$, and $\rho \in (0,\frac12)$. Sample $\Theta$ from $\set{0,1}$ such that $\Theta=1$ with probability $\rho$.
	Consider two arms with Bernoulli reward distributions from the following family:
	\begin{itemize}
		\item If $\Theta=1$, the \emph{Yes} case, where 
		\begin{enumerate}
			\item $\arm_1$ is with mean reward $\frac{1}{2}+\alpha$;
			\item $\arm_2$ is with mean reward $\frac{1}{2}+\alpha+\beta$.
		\end{enumerate}
		\item If $\Theta=0$, the mean rewards of $\arm_1$ and $\arm_2$ are both $\frac{1}{2}$.
	\end{itemize}
	Let $\ALG$ be an algorithm that uses at most $m=\frac{1}{16}\cdot \frac{\eps^3}{\rho \cdot (\alpha+\beta)^{2}}$ arm pulls on an instance $I$ sampled from the family. Let $\pi$ be the transcript of $\ALG$ that records the arm pulls and the results, and let $\Pi$ be the random variable of $\pi$. Then, with probability at least $1-\eps$ over the randomness of transcript $\Pi$, there is
	\begin{align*}
		& \Pr\paren{\Theta=1 \mid \Pi=\pi} \in [\rho -  \eps,  \rho +  \eps]\\
		& \Pr\paren{\Theta=0 \mid \Pi=\pi} \in [1-\rho - \eps,  1- \rho + \eps]
	\end{align*}
\end{lemma}

\begin{proof}
	We prove the lemma by an information-theoretic argument similar to the analysis in \cite{AW23BestArm}, albeit we need to handle the dependence between arm pulls in our case. For an $m$-trial process, let $\Pi=(\Pi^{1}, \Pi^{2}, \cdots, \Pi^{m})$, where $\Pi^{i}$ is the random variable for the transcript of the $i$-th arm pull. Therefore, we can bound the mutual information between $\Theta$ and $\Pi$ as follows. 
	\begin{align*}
		\II(\Theta; \Pi) &= \expectR{\theta \in \{0,1\}}{\kl{\Pi\mid \Theta=\theta}{\Pi}} \tag{KL-divergence view of mutual information}\\
		&= \expectR{\theta \in \{0,1\}}{\kl{(\Pi^{1}, \Pi^{2}, \cdots, \Pi^{m})\mid \Theta=\theta}{(\Pi^{1}, \Pi^{2}, \cdots, \Pi^{m})}}\\
		&= \expectR{\theta \in \{0,1\}}{\sum_{i=1}^{m} \kl{\Pi^{i} \mid (\Pi^{i+1}, \cdots,\Pi^{m}, \Theta=\theta)}{\Pi^{i} \mid (\Pi^{i+1}, \cdots,\Pi^{m})}} \tag{by chain rule of KL divergence}.
	\end{align*}
	We now argue that each of the KL-divergence terms in the expectation can be upper-bounded by substituting the transcript with the pull on $\arm_{2}$.
	\begin{claim}
		\label{clm:script-trial-ub}
		Let $\Pi^{i, \arm_2}$ be the random variable for the transcript induced by pulling $\arm_2$ on step $i$. For any $i\in [m]$ and $\theta \in \{0,1\}$, there is
		\begin{align*}
			& \kl{\Pi^{i} \mid \Pi^{i+1}, \cdots,\Pi^{m}, \Theta=\theta}{\Pi^{i} \mid \Pi^{i+1}, \cdots,\Pi^{m}} \\
			& \leq \rho \cdot \kl{\Pi^{i, \arm_1}\mid \Theta = \theta}{\Pi^{i, \arm_1}} + (1-\rho)\cdot \kl{\Pi^{i, \arm_2}\mid \Theta = \theta}{\Pi^{i, \arm_2}}.
		\end{align*}
	\end{claim}
	\begin{proof}
		The proof is similar to the one we showed in \Cref{clm:cross-trial-ub}. Concretely, let $J$ be the random variable for the choice of the arm to be pulled, and observe in the same manner as \Cref{clm:cross-trial-ub} that conditioning on the choice of $J$, the transcript between different $i$ indices are \emph{independent}. As such, For any $i\in [m]$ and $\theta\in\{0,1\}$, there is
		\begin{align*}
			& \kl{\Pi^{i} \mid \Pi^{i+1}, \cdots,\Pi^{m}, \Theta=\theta}{\Pi^{i} \mid \Pi^{i+1}, \cdots,\Pi^{m}} \\
			& \leq \kl{\Pi^{i} \mid \Pi^{i+1}, \cdots,\Pi^{m}, \Theta=\theta, J}{\Pi^{i} \mid \Pi^{i+1}, \cdots,\Pi^{m}, J}  \tag{extra conditioning can only increase KL-divergence}\\
			&= \kl{\Pi^{i} \mid  \Theta=\theta, J}{\Pi^{i} \mid J}. \tag{$\Pi^{i}$ is independent of $\Pi^{\neq i}$ conditioning on the choice of $J$}
		\end{align*}
		For the first random variable, we have
		\begin{align*}
			& \paren{\Pi^{i} \mid \theta=0, J=1} = \bern{1/2} \qquad \paren{\Pi^{i} \mid \theta=0, J=2} = \bern{1/2}; \\
			& \paren{\Pi^{i} \mid \theta=1, J=1} = \bern{1/2+\alpha} \qquad \paren{\Pi^{i} \mid \theta=1, J=2} = \bern{1/2+\alpha+\beta}.
		\end{align*}
		On the other hand, for the second random variable, there is
		\begin{align*}
			\paren{\Pi^{i} \mid J=1} = \bern{\frac{1}{2}+\rho \cdot \alpha}; \qquad \paren{\Pi^{i} \mid J=2} = \bern{\frac{1}{2}+\rho \cdot (\alpha + \beta)}.
		\end{align*}
		By the above calculation, the KL-divergences are maximized with $J=2$ for the $\Theta=0$ case and $J=1$ for $\Theta=1$ case. As such, we have
		\begin{align*}
			& \kl{\Pi^{i} \mid \Pi^{i+1}, \cdots,\Pi^{m}, \Theta=\theta}{(\Pi^{i} \mid \Pi^{i+1}, \cdots,\Pi^{m}} \\
			& \leq \kl{\Pi^{i} \mid J, \Theta=\theta}{\Pi^{i} \mid J}\\
			&= \expectR{j\in \{1,2\}}{\kl{\Pi^{i} \mid \Theta=\theta, J=j}{\Pi^{i} \mid J=j}}\\
			& \leq \rho \cdot \kl{\Pi^{i} \mid J=1}{\Pi^{i} \mid \Theta=\theta, J=1} + (1-\rho) \cdot \kl{\Pi^{i} \mid J=2}{\Pi^{i} \mid \Theta=\theta, J=2} \\
			&= \rho \cdot \kl{\Pi^{i, \arm_1}\mid \Theta = \theta}{\Pi^{i, \arm_1}} + (1-\rho)\cdot \kl{\Pi^{i, \arm_2}\mid \Theta = \theta}{\Pi^{i, \arm_2}},
		\end{align*}
		as desired. \myqed{\Cref{clm:script-trial-ub}}
	\end{proof}
	By \Cref{clm:script-trial-ub}, we can upper bound the mutual information between $\Theta$ and $\Pi$ as 
	\begin{align*}
		\qquad & \II(\Theta; \Pi) \\
		&\leq \expectR{\theta \in \{0,1\}}{\sum_{i=1}^{m} \kl{\Pi^{i, \arm_2}}{\Pi^{i, \arm_2} \mid \Theta = \theta}}\\
		&= \sum_{i=1}^{m} \rho \cdot \kl{\bern{\frac{1}{2}+\alpha}}{\bern{\frac{1}{2}+\rho\cdot \alpha}} + (1-\rho)\cdot \kl{\bern{\frac{1}{2}}}{\bern{\frac{1}{2}+\rho\cdot (\alpha+\beta)}}\\
		&\leq 8m\cdot \paren{\rho\cdot (\rho-1)^2\cdot \alpha^2 + (1-\rho) \cdot \rho^2 \cdot (\alpha+\beta)^2} \tag{by \Cref{clm:bernoulli-KL}}\\
		&\leq 16 m\cdot \rho \cdot (\alpha+\beta)^2 \tag{by $(\rho-1)^2\leq \rho^2$ since $\rho\leq \frac{1}{2}$}.
	\end{align*}
	By plugging in the condition that $m\leq \frac{1}{16} \cdot \frac{\eps^3}{\rho (\alpha+\beta)^2}$, we have $\II(\Theta; \Pi)\leq \eps^3$. We now use another KL-divergence form of the mutual information to get
	\begin{align*}
		\II(\Theta; \Pi) = \expectR{\pi\sim \Pi}{\kl{\Theta}{\Theta\mid \Pi=\pi}} \leq \eps^3.
	\end{align*}
	As such, with probability at least $1-\eps$ over the randomness of $\Pi$, we have
	\begin{align*}
		\kl{\Theta}{\Theta\mid \Pi=\pi}\leq \frac{1}{\eps} \cdot  \expectR{\pi\sim \Pi}{\kl{\Theta}{\Theta\mid \Pi=\pi}} \leq \eps^2.
	\end{align*}
	We condition on the high probability transcripts for the rest of the calculations. Now, we can apply Pinsker's inequality (\Cref{fact:pinsker}) to get 
	\begin{align*}
		\tvd{\Theta}{\Theta\mid \Pi=\pi} \leq \sqrt{\kl{\Theta}{\Theta\mid \Pi=\pi}} \leq \eps.
	\end{align*}
	By \Cref{fact:distinguish-tvd}, we get the desired upper bound of the ``advantage'', i.e. 
	\begin{align*}
		& \card{\Pr\paren{\Theta=0\mid \Pi=\pi}-\Pr\paren{\Theta=0}}\leq \eps\\
		& \card{\Pr\paren{\Theta=1\mid \Pi=\pi}-\Pr\paren{\Theta=1}}\leq \eps,
	\end{align*} 
	which implies the desired lemma statement.
\end{proof}

\subsection{Lower bounds on the Sample Complexity of MABs Trapping and Learning}
\label{subsec:batch-arm-hardness}
We now show how we `amplify' the result for the double-armed bandits to a collection of $k$ arms with two \emph{special} arms. These results are similar both in spirit and in technicality to existing multi-pass lower bounds \cite{AWneurips22,AW23BestArm}, and we include the proofs for completeness.

\begin{lemma}
	\label{lem:arm-trapping}
	Let $k\geq 3$ be an integer and $\alpha, \beta >0$ such that $\alpha+\beta<\frac{1}{6}$, suppose there is a family of $k$ arms in which
	\begin{itemize}
		\item two indices $\istar, \jstar \in [k]$ chosen uniformly at random (without replacement), and their mean rewards are $\mu_{\istar}=\frac{1}{2}+\alpha$ and $\mu_{\jstar}=\frac{1}{2}+\alpha+\beta$.
		\item for all $i \in [k]\setminus \{\istar, \jstar\}$, their mean rewards are $\mu_{i}=\frac{1}{2}$.
	\end{itemize}
	Then, for any given parameter $\tau\in (0, \frac{1}{2}]$, any algorithm that outputs $\frac{\tau\cdot k}{40}$ arms that contains any arm with reward \emph{strictly more than} $\frac{1}{2}$ with probability at least $\tau$ requires $\frac{1}{600}\cdot \frac{\tau^3}{(\alpha+\beta)^2}\cdot k$ arm pulls.
\end{lemma}
\begin{proof}
	\FloatBarrier
	The proof uses the ``direct-sum'' argument in a similar manner of \cite{AWneurips22} and \cite{AW23BestArm}. Concretely, we provide a reduction from the problem in \Cref{lem:arm-identify}, and show that an algorithm that satisfies the prescribed property in \Cref{lem:arm-trapping} with $s$ samples would imply an algorithm that identifies an arm with $O(s/k)$ samples, which eventually leads to a contradiction with \Cref{lem:arm-identify} for $\rho=1/2$. The formal reduction is as \Cref{red:arm-trapping}.
	\begin{algorithm}[!h]
		\caption{A reduction algorithm to prove \Cref{lem:arm-trapping}}\label{red:arm-trapping} 
		\KwIn{Two arms $\arm_{1}$ and $\arm_{2}$ from the distribution of \Cref{lem:arm-identify} with $\rho=\frac{1}{2}$.}
		\KwIn{An algorithm $\ALG$ that $a).$ uses at most $\frac{1}{600}\cdot \frac{\tau^3}{(\alpha+\beta)^2}\cdot k$ arm pulls, $b).$ outputs a collection $S$ of $\frac{\tau\cdot k}{20}$, and $c).$ $S$ contains an arm with reward \emph{strictly more than} $\frac{1}{2}$ with probability at least $\tau$.}
		\KwOut{The decision (Yes or No cases) from which $\arm_{1}$ and $\arm_{2}$ are sampled.}
		Sample a coin $\Theta \sim \bern{\frac{1}{2}+\frac{11}{38}\cdot \tau}$\;
		\If{$\Theta=0$}{
			Directly output ``$\arm_1$ is from $\bern{1/2+\alpha}$ and $\arm_2$ is from $\bern{1/2+\alpha+\beta}$''.\;
		}
		\Else{
			Construct an instance $I$: sample two indices $\istar$, $\jstar$ uniformly at random (without replacement), and set the arms with indices $\istar$, $\jstar$ as $\arm_{1}$ and $\arm_{2}$\;
			For all indices $i \in [k]\setminus \{\istar, \jstar\}$, create $k-2$ dummy arms $\bern{1/2}$\;
			Run $\ALG$ on instance $I$, and output with the following rules: \;
			\If{$\arm_1$ or $\arm_2$ uses more than $\frac{1}{30}\cdot \frac{\tau^2}{(\alpha+\beta)^2}$ arm pulls}{
				\label{line:sample-ub-term} Terminate $\ALG$ and output ``$\arm_1$ is from $\bern{1/2+\alpha}$ and $\arm_2$ is from $\bern{1/2+\alpha+\beta}$''. \;
			}
			\ElseIf{$S$ cotains any of $\{\istar, \jstar\}$}{
				\label{line:false-trap} Output ``$\arm_1$ is from $\bern{1/2+\alpha}$ and $\arm_2$ is from $\bern{1/2+\alpha+\beta}$''\;
			}
			\Else{
				Output ``$\arm_1$ and $\arm_2$ are from $\bern{1/2}$''\;
			}
		}
	\end{algorithm}
	
	We first observe that \Cref{red:arm-trapping} never uses more than $\frac{1}{30}\cdot \frac{\tau^2}{(\alpha+\beta)^2}$ arm pulls, as there is a forced termination once this condition happens. We now need to analyze the correctness of the algorithm for \Cref{lem:arm-identify}. Note that we fix $\rho$ in \Cref{lem:arm-identify} to be $\rho=\frac{1}{2}$. We claim that the algorithm correctly identifies the cases with probability at least $\frac12+\frac{\tau}{5}$, and the analysis considers two cases, respectively.
	\begin{enumerate}[label=\roman*).]
		\item $\arm_1$ is $\bern{1/2+\alpha}$ and $\arm_2$ is $\bern{1/2+\alpha+\beta}$. In this case, with probability $\frac{1}{2}-\frac{11}{38}\cdot \tau$, \Cref{red:arm-trapping} directly return the correct answer. On the other hand, if \Cref{red:arm-trapping} runs $\ALG$ on the instance $I$, it will return the correct answer as long as $\ALG$ succeeds, which is with probability at least $\tau$. As such, the correct probability \emph{conditioning} on the ``Yes'' case of \Cref{lem:arm-identify} is at least
		\begin{align*}
			\frac{1}{2}-\frac{11}{38}\cdot \tau + \left(\frac{1}{2}+\frac{11}{38}\cdot \tau\right)\cdot \tau \geq \frac{1}{2} +\left(\frac{1}{2}-\frac{11}{38}\right)\cdot \tau \geq \frac{1}{2}+ \frac{\tau}{5}.
		\end{align*}  
		\item Both $\arm_1$ and $\arm_2$ are $\bern{1/2}$. We show that if \Cref{red:arm-trapping} runs $\ALG$ on the instance $I$, the correct probability is sufficiently high. To this end, we bound the failure probability for \Cref{red:arm-trapping} to not report this case (conditioning on $\Theta=0$). Let $s_{(\ell)}$ be the random variables for the number of arm pulls used by the arm on index $\ell$. Furthermore, let us use $s_1$ and $s_2$ to denote the random variables for the number of samples used by $\arm_1$ and $\arm_2$. Let $\mathcal{E}_{\text{No}}$ be the event that $\arm_1$ and $\arm_2$ are sampled from the ``No'' case of \Cref{lem:arm-identify} (both $\bern{1/2}$). We have
		\begin{align*}
			\expect{s_1\mid \mathcal{E}_{\text{No}}} &= \expect{s_2\mid \mathcal{E}_{\text{No}}} \tag{$s_1$ and $s_2$ are identical random variables}\\
			&= \sum_{\ell=1}^{k} \Pr(i^* = \ell)\cdot \expect{s_{(\ell)}\mid \mathcal{E}_{\text{No}}} \tag{all arms are identical random variables conditioning on $\mathcal{E}_{\text{No}}$}\\
			&= \frac{1}{k}\cdot \expect{\sum_{\ell} s_{(\ell)}\mid \mathcal{E}_{\text{No}}}\\
			&\leq \frac{1}{600}\cdot \frac{\tau^3}{(\alpha+\beta)^2}. \tag{bound on the number of arm pulls}
		\end{align*}
		Therefore, we have $\Pr\paren{s_1\geq \frac{1}{30} \cdot \frac{\tau^2}{(\alpha+\beta)^2}}\leq \frac{\tau}{20}$ by a simple Markov bound. Therefore, the probability for \Cref{line:sample-ub-term} to falsely output the ``Yes'' case is at most $\frac{\tau}{4}$. On the other hand, conditioning on $\mathcal{E}_{\text{No}}$, the arms become identical random variables. More formally, let $X_{\arm_1}$ and $X_{\arm_2}$ be the indicator random variables for $\arm_1$ and $\arm_2$ to be in $S$, and let $X_{(\ell)}$ 
		be the indicator random variables for the arm of index $\ell$ to be in $S$, we have
		\begin{align*}
			\Pr(X_{(\ell)}=1) = \Pr(X_{\arm_{1}}=1) = \Pr(X_{\arm_{2}} = 1) = \frac{\card{S}}{k}\leq \frac{\tau}{40}.
		\end{align*} 
		Therefore, by a union bound, the algorithm to contain \emph{any} of $\arm_{1}$ and $\arm_{2}$ in the ``No'' case is at most $\frac{1}{20}$.
		Now, we apply another union bound, and the failure probability conditioning on $\mathcal{E}_{\text{No}}$ and $\ALG$ is executed on $I$ is at most $\frac{\tau}{20}+\frac{\tau}{20}=\frac{\tau}{10}$. As such, the success probability given $\mathcal{E}_{\text{No}}$ is at least
		\begin{align*}
			(\frac{1}{2}+\frac{11}{38}\cdot \tau)\cdot (1-\frac{\tau}{10}) \geq \frac{1}{2}+\frac{\tau}{5}.
		\end{align*}
	\end{enumerate}
	By plugging in $\eps=\frac{\tau}{5}$ and $\rho=\frac{1}{2}$ to \Cref{lem:arm-identify}, we obtain the number of necessary arm pulls is at least $\frac{1}{25}\cdot \frac{\tau^2}{(\alpha+\beta)^2}$ arm pulls, which forms a contradiction with \Cref{red:arm-trapping}. Therefore, such an $\ALG$ cannot exist.
	
	\FloatBarrier
\end{proof}

\begin{lemma}
	\label{lem:batch-arm-learning}
	Let $k\geq 3$ be an integer, $\alpha, \beta >0$ such that $\alpha+\beta<\frac{1}{6}$, and $\rho\in (0, \frac{1}{2})$, suppose there is a family of $k$ arms in which
	\begin{itemize}
		\item with probability $\rho$, the \emph{Yes} 
		case, where all except \emph{two} arms chosen uniformly at random are with mean rewards $\frac{1}{2}$, and the two special arms are with mean rewards $\frac{1}{2}+\alpha$ and $\frac{1}{2}+\alpha+\beta$.
		\item with probability $1- \rho$, the \emph{No} case, where all the arms are with mean rewards $\frac{1}{2}$.
	\end{itemize}
	Then, for any given parameter $\tau\in (0, \frac{1}{5}]$, let $\ALG$ be any algorithm that given an instance $D$ from the distribution, uses at most $\frac{1}{200}\cdot \frac{\tau^2}{\rho \cdot (\alpha+\beta)^2}\cdot k$ arm pulls, and let $\Pi$ and $\pi$ be the random variable and the realization of the transcripts of $ALG$. With probability at least $1-2 \tau^{1/2}$ over the randomness of the transcript, there is
	\begin{align*}
		& \Pr\paren{\text{$D$ in \emph{Yes} case} \mid \Pi=\pi} \in [\rho-2 \tau^{1/2}, \rho + 2 \tau^{1/2}];\\
		& \Pr\paren{\text{$D$ in \emph{No} case} \mid \Pi=\pi} \in [1-\rho-2 \tau^{1/2}, 1-\rho + 2 \tau^{1/2}],
	\end{align*}
	where the randomness is over the choices of the instances.
\end{lemma}
\begin{proof}
	\FloatBarrier
	Similar to the proof of \Cref{lem:arm-trapping} (and as in \cite{AW23BestArm}), we prove the lemma by applying the ``direct sum'' argument with \Cref{lem:arm-learn}. To this end, we again assume for the purpose of contradiction that an algorithm that breaks the bound on \Cref{lem:batch-arm-learning} exists, and build an algorithm that is ruled out by \Cref{lem:arm-learn}. In the proof, we only focus on the upper bound of $\Pr\paren{\text{$D$ in \emph{Yes} case} \mid \Pi=\pi}$ as the lower bound follows from the same logic.
	
	\begin{algorithm}[!h]
		\caption{A reduction algorithm to prove \Cref{lem:batch-arm-learning}}\label{red:batch-arm-learn}
		\KwIn{Two arms $\arm_{1}$ and $\arm_{2}$ from the distribution of \Cref{lem:arm-learn}.}
		\KwIn{An algorithm $\ALG$ that $a).$ uses at most $\frac{1}{200}\cdot \frac{\tau^2}{\rho\cdot (\alpha+\beta)^2}\cdot k$ arm pulls, $b).$ with probability more than $2\tau^{1/2}$ produce a transcript $\pi$, such that $\Pr\paren{\text{$D$ in \emph{Yes} case} \mid \Pi=\pi}> \rho + 2 \tau^{1/2}$.}
		\KwOut{A (conditional) probability distribution of instance $D$.}
		Construct an instance $\tilde{D}$: sample two indices $\istar$, $\jstar$ uniformly at random (without replacement), and set the arms with indices $\istar$, $\jstar$ as $\arm_{1}$ and $\arm_{2}$\;
		For all indices $i \in [k]\setminus \{\istar, \jstar\}$, create $k-2$ dummy arms $\bern{1/2}$\;
		Run $\ALG$ on instance $\tilde{D}$, and output with the following rules: \;
		\If{$\arm_1$ or $\arm_2$ uses more than $\frac{1}{5}\cdot \frac{\tau^{3/2}}{\rho \cdot(\alpha+\beta)^2}$ arm pulls}{
			\label{line:sample-ub-knowledge} Terminate $\ALG$ and output ``Yes case''. \;
		}
		\Else{$S$ contains any of $\{\istar, \jstar\}$}{
			\label{line:follow-knowledge} Output the distribution of the ``Yes'' and ``No'' cases of $\tilde{D}$ (which is a distribution generated by $\ALG$) as the distribution of the ``Yes'' and ``No'' cases of the problem in \Cref{lem:arm-learn}\;
		}
	\end{algorithm}
	
	The formal description of the reduction algorithm is as in \Cref{red:batch-arm-learn}. It is straightforward to observe that \Cref{red:batch-arm-learn} uses at most $\frac{1}{5}\cdot \frac{\tau^{3/2}}{\rho \cdot(\alpha+\beta)^2}$ arm pulls, as we terminate and output in \Cref{line:sample-ub-knowledge} otherwise. We now show that \Cref{red:batch-arm-learn} correctly ``learns'' the distribution of the arms with probability at least $2\tau^{1/2}$. To this end, we conduct the following case-based analysis:
	\begin{enumerate}[label=\roman*).]
		\item If the algorithm enters \Cref{line:sample-ub-knowledge}: we show that the algorithm reports ``Yes case'' correctly with probability at least $2\gamma^{1/2}$, which implies $\Pr\paren{\text{$D$ in \emph{Yes} case} \mid \Pi=\pi}=1\geq \rho + 2 \tau^{1/2}$. To see the desired statement, note that in the ``No case'', every arm becomes identical random variables. As such, similar to the proof of \Cref{lem:arm-trapping}, we can define $s_1$ and $s_2$ as the number of arm pulls used on $\arm_1$ and $\arm_2$, and show that
		\begin{align*}
			\expect{s_1 \mid \text{No case}} =\expect{s_2 \mid \text{No case}} \leq \frac{1}{200}\cdot \frac{\tau^2}{\rho\cdot (\alpha+\beta)^2}.
		\end{align*}
		As such, we have
		\begin{align*}
			\Pr\paren{s_1\geq \frac{1}{5}\cdot \frac{\tau^{3/2}}{\rho\cdot (\alpha+\beta)^2}} \leq \frac{\tau^{1/2}}{20}  \qquad \Pr\paren{s_2\geq \frac{1}{5}\cdot \frac{\tau^{3/2}}{\rho\cdot (\alpha+\beta)^2}} \leq \frac{\tau^{1/2}}{20}.
		\end{align*}
		Therefore, the probability for a transcript $\pi$ such that $\Pr\paren{\text{$D$ in \emph{Yes} case} \mid \Pi=\pi}=1$ is at least $1-\frac{\tau^{1/2}}{10}\geq 2\tau^{1/2}$ by the choice of $\tau\leq \frac{1}{5}$. 
		\item If the algorithm enters \Cref{line:follow-knowledge}, then by the guarantee of $\ALG$, we have that
		\begin{align*}
			\Pr_{\Pi}\paren{\Pr\paren{\text{$D$ in \emph{Yes} case} \mid \Pi=\pi}>\rho+2\tau^{1/2}} > 2\tau^{1/2}.
		\end{align*}
	\end{enumerate}
	Note that by using \Cref{lem:arm-learn} with $\eps=2\tau^{1/2}$, for the after mentioned bound to hold, at least $\frac{1}{2}\cdot \frac{\tau^{3/2}}{\rho\cdot (\alpha+\beta)^2}$ arm pulls are necessary. As such, it forms a contradiction with \Cref{red:batch-arm-learn}, which means such $\ALG$ cannot exist.
	
	\FloatBarrier
\end{proof}

\subsection{A Lower Bound Framework on Batched Distributions}

In our lower bound proof, we will crucially use a recent multi-pass lower bound tool developed by \cite{AW23BestArm}. The original lower bound construction of \cite{AW23BestArm} is on \emph{batched} instances distributions. On a high level, these distributions divide the arm into multiple batches with a fixed order. Inside each batch, most of the arms are ``flat'', i.e., with mean reward $\frac{1}{2}$, and one (or a few) \emph{special} arm(s) are planted uniformly at random among the indices. The reward distribution of the special arms is chosen randomly and independently between the reward of $\frac{1}{2}$ and $>\frac{1}{2}$. To make the instance hard, the distributions usually put batches whose special arm \emph{might} possess higher rewards to the late part of the stream. The intuition here is that to make sure the sample complexity upper bound is always followed, the streaming algorithm has to ``eliminate'' batches one by one in the reversed order of the stream.

The original analysis of \cite{AW23BestArm} was presented with only one specific distribution. In this section, we observe that their construction works for general batched distributions as long as they satisfy some properties. To this end, we formally define the \emph{batched} instances distributions.

\begin{definition}[Batched instance distributions]
	\label{def:batch-instance}
	Suppose the following information is given:  
	\begin{enumerate}[label=\roman*).]
		\item Positive integers $B\geq 2$, $C\geq 1$, $S \geq 1$; 
		\item A set of functions $F=\{f_{b}: \mathbb{N}^{+}\rightarrow (0,1)\}_{b=1}^{B+1}$ that computes $f_{b}(B)$ as a probability; 
		\item A set of tuples of positive real numbers $\mathrm{H} = \{(\etaib{1}{b}, \etaib{2}{b}, \cdots, \etaib{S}{b})\}_{b=1}^{B+1}$ from $(0, \frac12)$, and the values are (potentially) functions of $C$.
	\end{enumerate}
	We say instance distribution $\cD(B,C,F, \mathrm{H})$ is a $(B+1)$-\emph{batched instance distribution} if it satisfies the following properties:
	\begin{enumerate}[label=\alph*).]
		\item The arms are divided into $(B+1)$ batches;
		\item Inside each batch $b$, sample $S$ arms uniformly at random, and call them the \emph{special arms};
		\item All arms that are \emph{not} among the special arms follow the reward distribution $\bern{\frac{1}{2}}$.
		\item Sample a coin $\Theta_{b}\in \{0, 1\}$ from the distribution $\bern{f_{b}(B)}$ for each batch $b\in [B+1]$ \emph{independently}:
		\begin{itemize}
			\item If $\Theta_{b}=0$, set the special arms of batch $b$ with distribution $\bern{\frac{1}{2}}$.
			\item If $\Theta_{b}=1$, set the special arms of batch $b$ with distributions $\bern{\frac{1}{2}+\etaib{1}{b}}, \bern{\frac{1}{2}+\etaib{2}{b}}, \cdots, \bern{\frac{1}{2}+\etaib{S}{b}}$ (following an arbitrarily fixed order).
		\end{itemize}
	\end{enumerate}
\end{definition}

An illustration of batched instance distributions can be found in \Cref{fig:batched-instance}. For any $(B+1)$-batched instance distribution, it follows from the analysis of \cite{AW23BestArm} that the following proposition holds.

\begin{proposition}[\cite{AW23BestArm}, rephrased]
	\label{prop:multi-pass-lb}
	For a batched instance distribution $\cD(B,C,F, \mathrm{H})$ under the streaming setting, let the batches be arranged in the \emph{reversed} order of the stream arrival, i.e., $\mathcal{B}_{B+1}$ arrives first, and $\mathcal{B}_{1}$ arrives the last. Let the set of functions in $F$ be satisfying: 
	\begin{align*}
		f_{b}(B)=
		\begin{cases}
			\frac{1}{2B},\quad b\leq B;\\
			1, \quad b=B+1.
		\end{cases}
	\end{align*}
	For every batch $b\in [B+1]$, suppose w.log. that $\etaib{1}{b}\geq \etaib{s}{b}$ for any $s\in [S]$. Additionally, suppose $\cD(B,C,F, \mathrm{H})$ satisfies the following properties:
	\begin{itemize}
		\item \textbf{C1:} For each batch $\mathcal{B}_{b}$, \emph{conditioning on} $\Theta_{b}=1$, if an \emph{offline} algorithm uses at most $\frac{1}{700}\cdot \frac{\tau^3}{(\etaib{1}{b})^2}\cdot \frac{n}{B+1}$ arm pulls and outputs a collection of $\frac{1}{20}\cdot \frac{n}{B+1}\cdot \tau$ arms from $\mathcal{B}_{b}$, the probability for the output to contain any arm with reward \emph{strictly more than} $\frac12$ is at most $\tau$. 
		\item \textbf{C2:} For each batch $\mathcal{B}_{b}$, let $\nu$ be the probability for $\Theta_{b}=1$ (possibly conditioning on the information the algorithm obtains). Suppose we additionally obtain a new transcript $\pi$ with at most $\frac{1}{200}\cdot \frac{\tau^2}{\rho \cdot \left(\etaib{1}{b}\right)^2}\cdot \frac{n}{B+1}$ arm pulls. Then, with probability at least $1-2\tau^{1/2}$ over the randomness of $\pi$, there is
		\begin{align*}
			& \Pr\paren{\Theta_{b}=1 \mid \Pi=\pi} \in [\nu-2 \tau^{1/2}, \nu + 2 \tau^{1/2}];\\
			& \Pr\paren{\Theta_{b}=0 \mid \Pi=\pi} \in [1-\nu-2 \tau^{1/2}, 1-\nu + 2 \tau^{1/2}],
		\end{align*}
		\item \textbf{C3:} For any $b, p$ such that $p>b$, there is $\etaib{1}{p}\leq (\frac{1}{6C B})^{15}\cdot \etaib{1}{b}$.
	\end{itemize}
	Let $\ALG$ be a deterministic streaming algorithm that uses $P\leq B$ passes and a memory of at most $\frac{1}{30000}\cdot \frac{n}{B^3}$ arms. Additionally, suppose $\ALG$ satisfies
	\begin{equation}
		\label{equ:batch-sample-ub}
		\expect{\smp\mid \Theta_{b}=1, \Theta_{<b}=0}\leq C \cdot B^2 \cdot \frac{n}{\left(\etaib{1}{b}\right)^2}.
	\end{equation}
	Then, the probability for $\ALG$ to return the best arm is strictly less than $\frac{999}{1000}$.
\end{proposition}

\Cref{prop:multi-pass-lb} summarizes all the necessary conditions used by \cite{AW23BestArm}: in the analysis of \cite{AW23BestArm}, conditions \textbf{C1} and \textbf{C2} are used for the analysis the small-size sample case, and condition \textbf{C3} is used in the large-size sample case. Also, in the statement, there are two minor differences between \Cref{prop:multi-pass-lb} and the original theorem statement in \cite{AW23BestArm}:
\begin{enumerate}
	\item In \cite{AW23BestArm}, the result is only stated with $P=B$, i.e., using the number of passes directly as the parameter $B$. Here, we use $P\leq B$ since we work with an upper bound of $P$ that is only dependent on $n$.
	\item In \cite{AW23BestArm}, \Cref{equ:batch-sample-ub} does not have the $B^2$ factor. Nevertheless, it is evident from their proofs that we can add a $B$ factor on the sample bound.
\end{enumerate}

\section{Lower Bound: A Sharp Memory-pass Trade-off for Multi-pass Algorithms with Known $\Delta_{[2]}$}
\label{sec:lb-main}

We now introduce the construction and analysis of our main lower bound. 
Our adversarial instance follows the structure of batched instances as in \Cref{def:batch-instance}. On a high level, our instances keep \emph{two} special arms in each batch $b$ with stochastic mean rewards of either $\left(\frac{1}{2}, \frac{1}{2}\right)$ or $\left(\frac{1}{2}+\etaib{1}{b}, \frac{1}{2}+\etaib{2}{b}\right)$. In the latter case, which happens with probability roughly $O(1/B)$, we insist on \emph{invariate} $\etaib{1}{b}-\etaib{2}{b}$, which limits the utility for the knowledge of $\Delta_{[2]}$. 
Furthermore, we carefully pick the parameters such that the gap between $C\cdot \frac{n}{\left(\etaib{1}{b}\right)^2}$ becomes $\polylog{n}$. Since we only work with a number of passes of $\Theta(\log(n)/\log\log(n))$, the construction allows us to ``reduce'' the $O\left(\sum_{i=2}^{n}\frac{1}{\Delta^2_{[i]}}\right)$ 
sample complexity to the $C\cdot \frac{n}{\left(\etaib{1}{b}\right)^2}$ bound, which in turn allows us to use \Cref{prop:multi-pass-lb} to establish the lower bound.

We now give the formal construction of the instance family.

\begin{tbox}
	$\cP(B, C, \gamma)$: A hard instance distribution for multi-pass MABs algorithms with known $\Delta_{[2]}$. 
	
	\begin{enumerate}
		\item \textbf{Parameters}: Ensure that $\frac{1}{20}\cdot \frac{1}{n^{1/3}}\leq \gamma \leq \frac{1}{10}\cdot \frac{1}{n^{1/3}}$, and let $\chi_{1}=n^{1/3}\cdot \gamma$; furthermore, for any $b\in [B]$, let 
		\[\chi_{b+1} = \paren{\frac{1}{12 C \log(n)}}^{15} \cdot \chi_{b}.\]
		\item \textbf{Division of arms:} Divide the $n$ arms into $(B+1)$ batches of equal sizes, and put them in the \emph{reverse} order of the stream, i.e. $\mathcal{B}_{B+1}$ arrives first, and $\mathcal{B}_{1}$ arrives the last.
		\item \textbf{Sampling special arms: } For each batch $b\in [B+1]$, sample \emph{two} arms uniformly at random (without replacement), and call them \emph{special arms}. Set all the arms \emph{except} the special arms with reward distribution $\bern{1/2}$.
		\item \textbf{Batches $b\in[B]$: } For each $b\in [B]$, sample $\Theta_{b}$ from distribution $\bern{1/2B}$:
		\begin{enumerate}
			\item If $\Theta_{b}=0$, set both special arms with reward distributions $\bern{1/2}$.
			\item Otherwise, if $\Theta_{b}=1$ 
			\begin{itemize}
				\item Set the first special arm with reward distribution $\bern{1/2+\chi_{b}}$.
				\item Set the second special arm with reward distribution $\bern{1/2+\chi_{b}+\gamma}$.
			\end{itemize}
		\end{enumerate}
		\item \textbf{The batch $B+1$: } Always set the reward distributions of the special arms as follows ($\Theta_{B+1}=1$ deterministically) 
		\begin{itemize}
			\item Set the first special arm with reward distribution $\bern{1/2+\chi_{B+1}}$.
			\item Set the second special arm with reward distribution $\bern{1/2+\chi_{B+1}+\gamma}$.
		\end{itemize}
	\end{enumerate}
\end{tbox}

\begin{figure}
	\centering
	\begin{subfigure}{0.95\textwidth}
		\centering
		\includegraphics[scale=0.18]{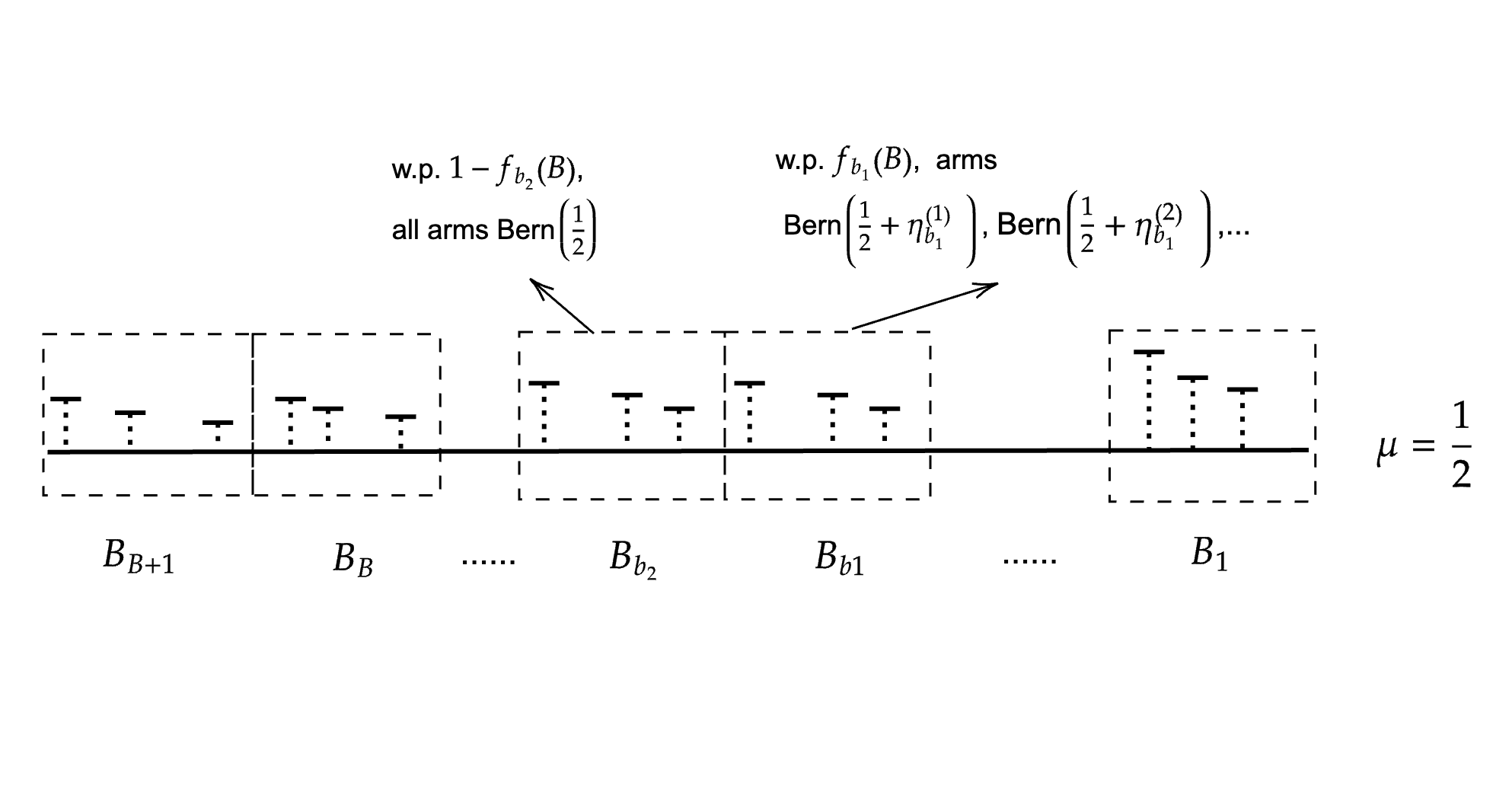}
		\caption{General $(B+1)$-Batched Instance Distribution}
		\label{fig:batched-instance}
	\end{subfigure}%
	
	\begin{subfigure}{0.95\textwidth}
		\centering
		\includegraphics[scale=0.18]{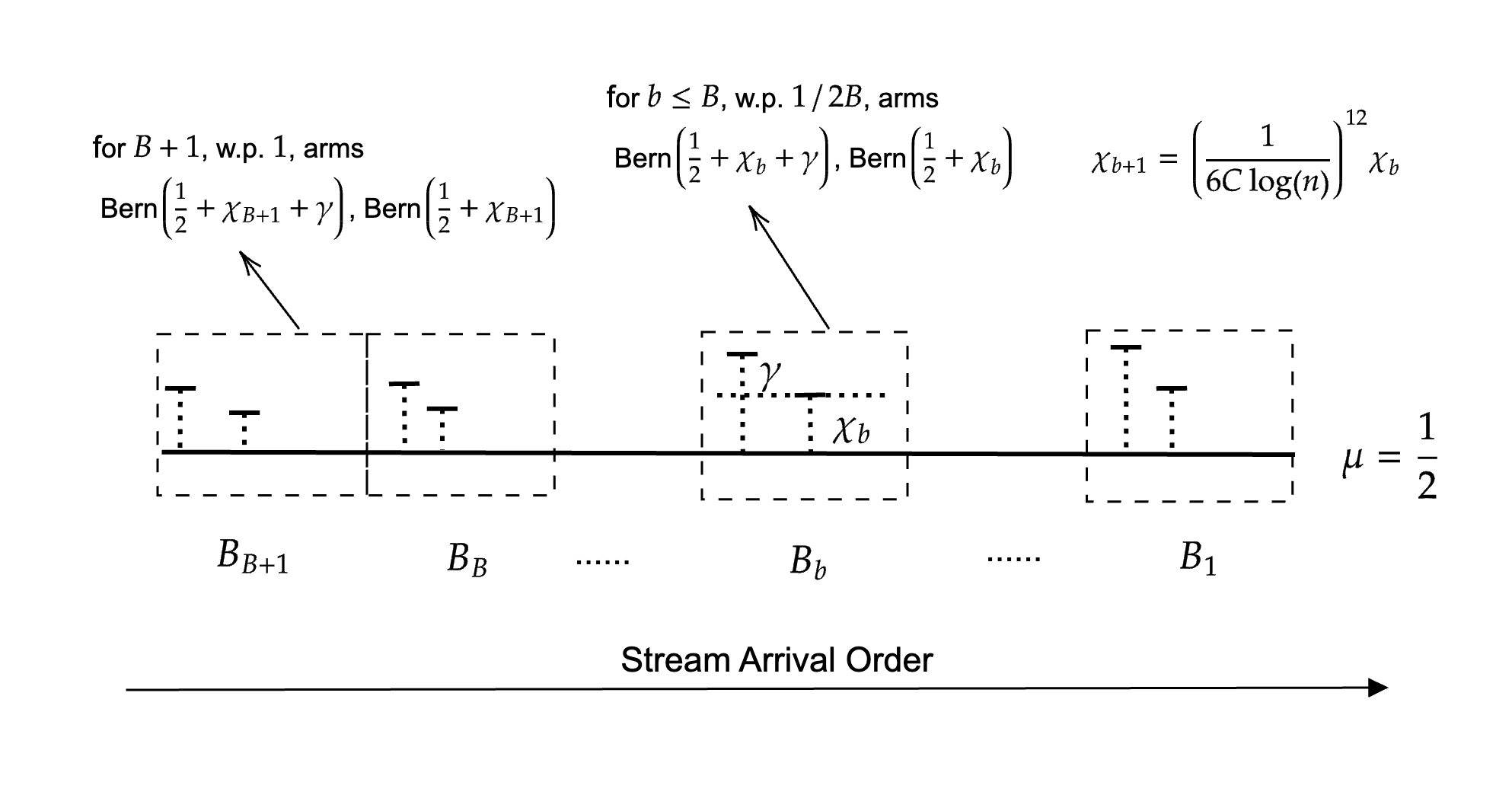}
		\caption{$\cP(B, C, \gamma)$ Instance distribution}
		\label{fig:streaming-adv-instance}
	\end{subfigure}
	\caption{An illustration of the general $(B+1)$-batched instance distribution (\Cref{def:batch-instance}) and the $\cP(B, C, \gamma)$ instance distribution. The mean rewards of arms are ranked in the decrement order from left to right for illustration purposes -- their positions inside the batches are uniformly at random.}
	\label{fig:lb-instance-illus}
\end{figure}

An illustration of the distribution $\cP(B, C, \gamma)$ can be shown as \Cref{fig:streaming-adv-instance}. It is straightforward to observe that the $\cP(B, C, \gamma)$ family follows the $(B+1)$-batched instance as in \Cref{def:batch-instance}. More concretely, in $\cP(B, C, \gamma)$, the arms are divided into $(B+1)$ batches, we have $S=2$ and the values of $\etaib{i}{b}$ as functions of $C$, and the probability functions are $f_{b}(B)=\frac{1}{2B}$ for all $b \in [B]$ and $f_{B+1}(B)=1$. Furthermore, we make the crucial observation that $\Delta_{[2]}$ is invariant across different settings.

\FloatBarrier

\begin{observation}
	\label{obs:Delta-invariate}
	For any instance in $\cP(B, C, \gamma)$, the value of $\Delta_{[2]}$ is equal to $\gamma$. In other words, in $\cP(B, C, \gamma)$, for all $b\in[B+1]$, there is 
	\begin{align*}
		\paren{\Delta_{[2]} \mid \Theta_{<b}=0, \Theta_{b}=1} = \gamma.
	\end{align*} 
\end{observation}


We now use $\cP(B, C, \gamma)$ to state our main multi-pass lower bound. 
\begin{theorem}[Formalization of \Cref{rst:main-lb}]
	\label{thm:lb-main}
	There exists a family of streaming MABs instances $\cP$, such that any streaming algorithm (deterministic or randomized) that given the quantity of $\Delta_{[2]}$, finds the best arm from an instance sampled from $\cP$ with an \emph{expected} sample complexity of $O\paren{\sum_{i=2}^{n} \frac{1}{\Delta^2_{[i]}} \cdot \log(n)}$, a success probability of at least $1999/2000$, and a memory of $o\paren{n/\log^3 {n}}$ arms has to make $\Omega\paren{\frac{\log(n)}{\log\log(n)}}$ passes over the stream. 
\end{theorem}

To prove \Cref{thm:lb-main}, the rest of this section is dedicated to two parts. We first show that the family of $\cP(B, C, \gamma)$ \emph{with $B=\Theta(\frac{\log n}{\log\log n})$} satisfies the conditions characterized by \Cref{prop:multi-pass-lb}. As such, to ensure the success probability is high, any algorithm must break the sample upper bound of \Cref{equ:batch-sample-ub}. Subsequently, we show that to keep the expected sample complexity of $O(\sum_{i=2}^{n}\frac{1}{\Delta^2_{[i]}})$, 
the sampling upper bound of \Cref{equ:batch-sample-ub} has to be satisfied, which forms a contradiction for the proof of \Cref{thm:lb-main}.

More concretely, the first part of the argument can be summarized as \Cref{lem:hard-B-dist}. (Note that in the lemma, we do not assume the knowledge of $\Delta_{[2]}$, as we will deal with it in the proof of \Cref{thm:lb-main} later.)
\begin{lemma}
	\label{lem:hard-B-dist}
	Let $C\geq 1$ be a fixed integer and $\gamma \leq \frac{1}{10}\cdot \frac{1}{n^{1/3}}$ be a real number. Let $B=\frac{1}{100C}\cdot \frac{\log n}{\log\log(n)}$, and let $\ALG$ be any deterministic $P$-pass streaming algorithm such that $P\leq B$. Suppose that $\ALG$ uses a memory of at most $\frac{1}{30000}\cdot \frac{n}{B^3}$ arms. 
	Additionally, suppose on instances of distribution $\cP(B, C, \gamma)$ and every $b \in [B+1]$, $\ALG$ satisfies:
	\[
	\Exp\bracket{\smp \mid\Theta_{b}=1, \Theta_{<b}=0} \leq C \cdot B^2 \cdot \frac{n}{(\chi_{b}+\gamma)^2}, 
	\]
	where the randomness is taken over the choice of the instance $I \sim \cP(B, C, \gamma) \mid \Theta_{b}=1, \Theta_{<b}=0$. 
	Then, the probability that $\ALG$ can output the best arm for $I \sim \cP(B, C, \gamma)$ is strictly less than $999/1000$.  
\end{lemma}
\begin{proof}
	We prove the lemma by showing that the distribution $\cP(B, C, \gamma)$ satisfied the conditions prescribed by \Cref{prop:multi-pass-lb}, which will allow us to directly use the conclusion therein. To this end, we verify \textbf{C1}, \textbf{C2}, and \textbf{C3}, respectively:
	\begin{itemize}
		\item Condition \textbf{C1}. We use \Cref{lem:arm-trapping} to argue this property. For any batch $b\in [B+1]$, we use $\alpha=\chi_{b}$ and $\beta=\gamma$. Since we set $\gamma\leq \frac{1}{10}\cdot \frac{1}{n^{1/3}}$, there is clearly $\chi_{b}+\gamma<\frac{1}{6}$. Furthermore, we let $k=\frac{n}{B+1}$ as the number of arms in each batch. Suppose for the purpose of contradiction that \textbf{C1} does \emph{not} hold. By our construction, we have $\etaib{1}{b}=\chi_{b}+\gamma$, and the assumption implies an algorithm that
		\begin{enumerate}
			\item uses at most $\frac{1}{700}\cdot \frac{\tau^3}{(\chi_{b}+\gamma)^2}\cdot \frac{n}{B+1}<\frac{1}{600}\cdot\frac{\tau^3}{(\alpha+\beta)^2}\cdot k$ arm pulls;
			\item outputs a collection of $\frac{1}{20}\frac{\tau n}{B+1}=\frac{\tau \cdot k}{20}$ arms, in which contains an arm with reward strictly more than $\frac{1}{2}$ with probability at least $\tau$,
		\end{enumerate}
		which forms a contradiction with \Cref{lem:arm-trapping}. Therefore, the condition \textbf{C1} has to be satisfied.
		
		\item Condition \textbf{C2}. We use \Cref{lem:batch-arm-learning} to verify this property. Again, we use $\alpha=\alpha_{b}$ and $\beta=\gamma$. Furthermore, let $\mathcal{E}$ be any event we want to condition on, and we let $\nu=\Pr\paren{\Theta_{b}=1\mid \mathcal{E}}$ be the probability for the distribution in the \emph{yes} case from the algorithm's internal view. Now, to use \Cref{lem:batch-arm-learning}, we simply set $\rho=\nu$, and if condition \textbf{C2} is not satisfied, the output distribution of the transcripts will violate \Cref{lem:batch-arm-learning}. Thus, condition \textbf{C2} must be followed in $\cP(B, C, \gamma)$.
		
		\item Condition \textbf{C3}. Note that we have $B= \frac{1}{100C}\cdot \frac{\log n}{\log\log n}$. As a result, we also have 
		\begin{equation}
			\label{equ:a-B-lower-bound}
			\begin{aligned}
				\chi_{b}\geq \chi_{B} & = \gamma \cdot n^{1/3} \cdot \paren{\frac{1}{12 C\log n}}^{\frac{10\log n}{100 C\log\log(n)}}\\
				&= \gamma \cdot n^{1/3} \cdot \paren{\frac{1}{12C}}^{\frac{\log n}{10C \log\log(n)}} \cdot \paren{\frac{1}{\log{n}}}^{\frac{\log n}{10C \log\log(n)}}\\
				& \geq \gamma \cdot n^{1/3}\cdot \paren{\frac{1}{2}}^{^{\frac{\log n}{\log\log(n)}}}\cdot \paren{\frac{1}{\log n}}^{\frac{\log n}{10C \log\log(n)}}\\
				& = \gamma \cdot n^{1/3}\cdot \frac{1}{n^{1/10C + o(1)}}\\
				& \geq \gamma\cdot n^{1/5},
			\end{aligned}
		\end{equation}
		where the second inequality if because $(\frac{1}{12C})^{\frac{1}{10C}}\geq \frac{1}{2}$ for any $C\geq 1$.
		Therefore, we have $\gamma\leq n^{-1/5}\chi_{b}$ for any choice of $\gamma$. As such, for any $r>b$, there is 
		\begin{align*}
			\frac{\etaib{1}{r}}{\etaib{1}{b}} &\leq \frac{\chi_{b+1}+\gamma}{\chi_{b}+\gamma}\\
			&\leq \frac{\chi_{b+1}\cdot \log n}{\chi_{b}} \tag{by $\gamma\leq n^{-1/5}\chi_{b}$ for sufficiently large $n$}\\
			&\leq \paren{\frac{1}{12 C \log(n)}}^{15} \\
			&\leq \paren{\frac{1}{6 B C}}^{15}, \tag{by $B\leq \log{n}$}
		\end{align*}
		which verifies the validity of condition \textbf{C3}.
	\end{itemize} 
	Finally, we observe that the memory and sample bound in \Cref{lem:hard-B-dist} matches the bound for deterministic algorithms in \Cref{prop:multi-pass-lb}, concluding the proof. 
\end{proof}

We are now ready to wrap up the proof of our main lower bound of \Cref{thm:lb-main}.

\begin{proof}[Proof of \Cref{thm:lb-main}]
	We focus on deterministic algorithms in the proof with success probability $\frac{999}{1000}$. The lower bound for randomized algorithms can be obtained by an application of Yao's minimax principle.
	
	We first deal with algorithms that do \emph{not} have the \emph{a priori} knowledge of $\Delta_{[2]}$. Assume the purpose of contradiction that there exists a $P$-pass streaming algorithm that uses
	\begin{enumerate}
		\item A memory of at most $\frac{1}{20000}\cdot \frac{n}{\log^3 n}$ arms;
		\item A success probability of at least $\frac{999}{1000}$;
		\item A sample complexity of $C' \cdot \log(n)\cdot \sum_{i=2}^{n}\frac{1}{\Delta^2_{[i]}}$;
		\item The number of passes $P$ satisfies $P\leq \frac{1}{100C}\cdot \frac{\log(n)}{\log\log(n)}$.
	\end{enumerate}
	Then, we can pick $C=2C'$ to construct the hard family $\mathcal{P}(\frac{1}{100C}\cdot \frac{\log(n)}{\log\log(n)}, C, \gamma)$ (pick any suitable $\gamma$). Observe that following the same calculation of \Cref{equ:a-B-lower-bound} and by the property of the distribution, for any $b\in[B+1]$, there is
	\begin{align*}
		\chi_{b}\geq \chi_{B} & = \gamma \cdot n^{1/3} \cdot \paren{\frac{1}{12 C\log n}}^{\frac{10\log n}{100 C\log\log(n)}}\\
		&= \gamma \cdot n^{1/3} \cdot \paren{\frac{1}{12C}}^{\frac{\log n}{10C \log\log(n)}} \cdot \paren{\frac{1}{\log{n}}}^{\frac{\log n}{10C \log\log(n)}}\\
		& \geq \gamma \cdot n^{1/3}\cdot \paren{\frac{1}{2}}^{^{\frac{\log n}{\log\log(n)}}}\cdot \paren{\frac{1}{\log n}}^{\frac{\log n}{10C \log\log(n)}}\\
		& = \gamma \cdot n^{1/3}\cdot \frac{1}{n^{1/10C + o(1)}}\\
		& \geq \gamma\cdot n^{1/5}.
	\end{align*}
	Furthermore, by the upper and lower bound on $\gamma$, we have that for $b\in[B+1]$, there is
	\begin{align*}
		\frac{1}{\gamma^2} &\leq 400 \cdot n^{2/3} \tag{by $\gamma\geq \frac{1}{20}\cdot \frac{1}{n^{1/3}}$}\\
		&\leq 50 \cdot n \tag{for sufficiently large $n$}\\
		&\leq \frac{n}{(\chi_{1}+\gamma)^2} \tag{$\chi_{1}\leq \frac{1}{10}$ and $\gamma \leq \frac{1}{n^{1/5}}\chi_{1}$}\\
		&\leq \frac{n}{(\chi_{b}+\gamma)^2} = \frac{n}{(\etaib{1}{b})^2}. \tag{$\chi_{b}\leq \chi_{1}$} 
	\end{align*}
	As such, it is straightforward to see that
	\begin{align*}
		\paren{C'\cdot \log(n) \cdot \sum_{i=2}^{n}\frac{1}{\Delta^2_{[i]}} \middle| \Theta_{b}=1, \Theta_{<b}=0} &\leq C'\cdot \log(n) \cdot \left(\frac{1}{\gamma^2} + \frac{n}{(\etaib{1}{b})^2}\right)\\
		&\leq 2 C'\cdot \log(n)\cdot \frac{n}{(\etaib{1}{b})^2}.
	\end{align*}
	Therefore, if a deterministic algorithm $\ALG$ satisfies: 
	\[\Exp\bracket{\smp} \leq C' \cdot \log(n)\cdot \sum_{i=2}^{n}\frac{1}{\Delta^2_{[i]}},\] it implies that the algorithm has sample complexity of 
	\[\Exp\bracket{\smp \mid \Theta_{b}=1, \Theta_{<b}=0}\leq C \cdot \log(n) \cdot \frac{n}{(\etaib{1}{b})^2}.\]
	This leads to a contradiction with \Cref{lem:hard-B-dist}, which implies that such a streaming algorithm cannot exist.
	
	Finally, to complete the proof, we note that on the constructed family $\mathcal{P}(\frac{1}{100C}\cdot \frac{\log(n)}{\log\log(n)}, C, \gamma)$, there is always $\Delta_{[2]}=\gamma$ by \Cref{obs:Delta-invariate}. As such, if we have a streaming algorithm $\widetilde{\ALG}$ that only works with a \emph{known} $\Delta_{[2]}$, we can simulate the algorithm without this prior knowledge by running $\widetilde{\ALG}$ as an inner streaming algorithm with parameter $\Delta_{[2]}=\gamma$. The contradiction still holds, concluding the proof.
\end{proof}

\begin{remark}
	By Remark 5.8 of \cite{AW23BestArm}, the $B^2$ term in \Cref{equ:batch-sample-ub} could be made to $B^C$ for any fixed constant $C$ with larger gaps between $\etaib{1}{b}$ values for different $b\in[B]$. Therefore, we could strenghthen the sample complexity lower bound in \Cref{thm:lb-main} to $O\paren{\sum_{i=2}^{n} \frac{1}{\Delta^2_{[i]}} \cdot \polylog(n)}$.
\end{remark}


\section{Upper Bound: A Multi-pass Pure Exploration Streaming MABs Algorithm with Known $\Delta_{[2]}$}\label{sec:basic}
A natural question to follow from our lower bound in \Cref{sec:lb-main} is whether this bound is tight.
In this section, we show our main upper bound result that nearly matches our lower bound in \Cref{sec:lb-main}. In particular, we prove the following theorem.
\begin{theorem}[Formalization of \Cref{rst:main-ub}]
	\label{thm:basic}
	For any $P\geq 1$, there exists a $(P + 1)$-pass streaming algorithm that given a streaming MABs instance and a known value of $\Delta_{[2]}$, finds the best arm with probability at least $1-\delta$ with a single-arm memory and at most \[O\left(\log \left(\frac{n P}{\delta}\right) \sum_{i = 2}^n \frac{n^{2/P}}{\Delta^2_{[i]}}\right)\]
	arm pulls. 
\end{theorem}

Note that by plugging in $P=\Theta(\log(n))$, \Cref{thm:basic} gives an $O(\log(n))$-pass algorithm with $O(\sum_{i = 2}^n \frac{1}{\Delta^2_{[i]}}\cdot \log{n})$ sample complexity, as we have stated in \Cref{rst:main-ub}.

We describe the general algorithm that could `adjust' the sample complexity bound with the number of passes as a parameter, see \Cref{alg:main}. 
Our approach is similar to the elimination algorithm in~\cite{KarninKS13}, and by leveraging the information about \(\Delta_{[2]}\), we can significantly reduce the number of passes. 
The algorithm proceeds in \(P  +1\) passes over the stream, and maintains a set $I_p$ of `active arms' on pass $p$. In pass \(p\), it samples each arm in the current set \(I_p\) for a carefully chosen \(T_p\) times. After sampling, it computes the estimated mean reward \(\hat{\mu}^{p}_i\) for each arm \(i\) in \(I_p\), and find the maximum estimated mean \(\hat{\mu}^{p}_{\max}\) among the arms in $I_p$. The algorithm then constructs a new set \(I_{p + 1}\) by eliminating arms whose estimated means are more than \(\epsilon_p\) below \(\hat{\mu}^{p}_{\max}\). After \(P + 1\) passes, if the set \(I_{P + 1}\) contains a single arm, the algorithm returns that arm as the best arm. 


\begin{algorithm}
    \caption{The Main Multi-pass Streaming Algorithm}\label{alg:main}
    \KwIn{Stream \(I\), parameter \(P\), gap parameter \(\Delta_{[2]}\), and confidence parameter \(\delta\)}
    \KwOut{Best arm}
    Set \(n \gets \abs{I}\) and \(I_0 \gets \{1, \dotsc, n\}\)\;
    Let \(\epsilon_p = n^{1-p/P}\Delta_{[2]} / 4\) for \(p = 0, \dotsc, P\) \;
    \For{\(p = 0, \dotsc, P\)}{
        \ForEach{\(i \in I\) in the arrival order}{
            \If{\(i \not\in I_p\)}{
                Skip arm\;
            }
            Pull arm \(i\) until the number of pulls reach \(T_p \triangleq \frac{8 \log(2n (P + 1) / \delta)}{\epsilon^2_r \log e}\) times\;
            Compute estimated mean \(\hat{\mu}^p_i\) after \(T_p\) pulls\;
        }
        Pick \(\hat{\mu}^p_{\max} = \max\limits_{i \in I_r}\{\hat{\mu}^p_i\}\)\;
        Create a new set \(I_{p + 1} \gets \{i \in I_p \mid \hat{\mu}^p_i \ge \hat{\mu}^p_{\max} - {\epsilon_p}\}\)\;
    }
    \If{\(I_{P + 1}\) contains one element}{
        \Return{single index of arm from \(I_{P + 1}\)}\;
    }
\end{algorithm}

It is easy to observe that the streaming algorithm (\Cref{alg:main}) maintains only a single arm memory across the passes. Formally:
\begin{lemma}\label{lem:main-alg-memory}
The memory of \Cref{alg:main} contains at most one arm at any point of the stream.
\end{lemma}
\begin{proof}
The algorithm initializes the set \(I_0\) to contain all the arm indices from the stream \(I\). However, this set doesn't actually store the arms themselves, only their indices. The arms are fetched one by one during the execution. 

For each pass number \(p\), from \(0\) to \(P\), the algorithm iterates through each arm \(i\) in the stream \(I\) and loads it to the memory. During this iteration, the algorithm pulls from it a fixed number of times, updates statistics \(\hat\mu^{p}_i\), and then moves on to the next arm. Therefore, at any given time, the memory contains only one arm, which is the arm currently being proceeded. 
We note that the whole information \(\hat\mu^p_i\) and \(I_p\) can be extracted from the transcript \(\Pi\), and we maintain this information explicitly only for the convenience of the presentation.  
\end{proof}

What remains is to prove the correctness and the sample complexity of \Cref{alg:main}. For simplicity and to remove any issues with dependency, in the analysis, we apply the following standard trick. We assume that we first extract \(T_P\) samples from each arm before the algorithm even starts the work. During the work, \Cref{alg:main} only utilize these samples for the computation. In other words, uses a prefix of length \(T_p\) for computation \(\hat\mu^{p}_i\). 
This strategy provides a structured and consistent set of samples and values \(\hat\mu^{p}_i\) for the algorithm's decisions, avoiding any dependency from continuously updated statistics.

We first define a `good event' that captures the high-probability bound for the empirical rewards to deviate from the actual rewards with $T_{p}$ number of arm pulls. More formally, we define event $\cE$ as follows.

\begin{align}\label{eq:event-E}
    \E \triangleq \{\forall{i \in I}, p \in \{0, \dotsc, P\} : \abs{\hat\mu^p_i - \mu_i} \le \epsilon_r / 4\}
\end{align}

We show that $\E$ holds probability at least $(1-\delta)$, which is by an application of the Chernoff-Hoeffding bound (\Cref{lem:chernoff}). The formal lemma and proof are as follows.
\begin{lemma}\label{lem:bound-E}
    Let \(\E\) be the event defined in \Cref{eq:event-E}. Then, \[\Pr\left[~\neg \E~\right] \le \delta \,.\]
\end{lemma}
\begin{proof}
    By the union bound, we have:
    \begin{align}\label{eq:event-sum}
        \Pr\left[\neg {\E}\right] \le \sum_{i \in I} \sum_{p = 0}^{P} \Pr\left[\abs{\hat{\mu}^p_i - \mu_i} > \epsilon_p/4\right]\,.
    \end{align}
    
    By the Chernoff-Hoeffding inequality (\Cref{lem:chernoff}), we have:
    \begin{align}\label{eq:event-single}
        \Pr\left[\abs{\hat{\mu}^p_i - \mu_i} > \frac{\epsilon_p}{4} \right] \le 2 \exp\left(-\frac{\epsilon_p^2 T_p}{8}\right) \le 2\exp\left(-\ln\left(\frac{n(P + 1)}{\delta}\right)\right) \le \frac{\delta}{n (P + 1)}\,.
    \end{align}
    
    Combining~\Cref{eq:event-sum} and~\Cref{eq:event-single}, we get:
    \begin{equation*}
        \Pr\left[\neg \E\right] \le n (P + 1) \frac{\delta}{n(P + 1)} \le \delta\, ,
    \end{equation*}
    as desired.
\end{proof}

In the rest of the proof, we condition on the high probability event that $\E$ happens. We now establish the correctness of the algorithm by \Cref{lem:best} and \Cref{lem:eliminate-large-gap}. Our target is to demonstrate that the algorithm correctly identifies the best arm in stream \(I\). We prove that if the event \(\E\) holds, then the algorithm outputs the best arm. Combining this result with the bound from \Cref{lem:bound-E}, we get that the probability of the algorithm making an incorrect identification is at most \(\delta\), thus affirming the algorithm's correctness.

We first prove that the best arm cannot be eliminated during the work of the algorithm. 
\begin{lemma}\label{lem:best}
    Conditioning on the event \(\E\) defined in \Cref{eq:event-E} holds, for any \(p \in \{0, \dotsc, P + 1\}\), we have \(\star \in I_p\). 
\end{lemma}

\begin{proof}
    To prove that the best arm \(\star\) cannot be eliminated in the algorithm, we start by assuming the opposite. Assume that the best arm \(\star\) belongs to \(I_p\), but does not belong to \(I_{p + 1}\) for some \(r\). This implies the existence of an arm \(i\) for which the following inequality holds 
    \begin{equation}\label{eq:upper-bound}
        \hat{\mu}^{p}_\star < \hat{\mu}^{p}_i - \epsilon_p \,.
    \end{equation}

    However, from the property of the event \(\E\), we have \(\hat{\mu}^{p}_\star \ge \mu_\star - \epsilon_p / 4 \) and \(\hat{\mu}^p_i \le \mu_i + \epsilon_p / 4\).
    Consequently, we get: 
    \begin{equation}\label{eq:lower-bound}
        \hat{\mu}^p_\star - \hat{\mu}^p_i \ge \mu_\star - \mu_i -\epsilon_p / 2 \ge -\epsilon_p / 2\,.
    \end{equation}

    The~\Cref{eq:upper-bound} and~\Cref{eq:lower-bound} contradict each other, leading to a contradiction. Therefore, it is not possible for \(\star\) to be in \(I_p\) and eliminate the subsequent round, and as a result, \(\star\) remains in \(I_p\) for any \(p\). Thus, the best arm \(\star\) is always present in the set \(I_p\) for any \(p\) and consequently in set \(I_{P + 1}\). 
\end{proof}

Next, we demonstrate that if event \(\cE\) holds, then an arm with a large gap should be eliminated before a specific iteration in Algorithm~\ref{alg:main}. In other words, if the condition \(\cE\) is satisfied, the algorithm will identify and eliminate arms with significant gaps with a specific number of iterations. 

\begin{lemma}\label{lem:eliminate-large-gap}
    Conditioning on the event \(\cE\) defined in \Cref{eq:event-E} holds, and suppose for arm \(i\) and an integer $p\in [P+1]$, there is \(\Delta_i > \frac{3}{2}\epsilon_p\), then \(i \notin I_{p + 1}\).
\end{lemma}

\begin{proof}
    Consider any arm \(i\) and a value \(r\) such that \(\Delta_i > \frac{3}{2} \epsilon_p \). We aim to prove that arm \(i\) will not be included in the set \(I_{p  +1}\). We can first assume that \(i\in I_p\) since otherwise, if \(i\) is not in \(I_p\), arm \(i\) cannot be included in \(I_{p + 1}\) simply by set inclusion.

    Assuming \(i\) is in \(I_p\), by using \Cref{lem:best} and the event \(\cE\), we have the following inequality:
    \begin{equation}\label{eq:a1}
        \hat{\mu}^{p}_{\max} \geq \hat{\mu}^{p}_{\star} \geq \mu_\star - \epsilon_p / 4\,.
    \end{equation}

    This inequality indicates that the maximum estimated mean \(\hat{\mu}^{p}_{\max}\) in \(p\)-th iteration is at least as large as the estimated mean \(\hat{\mu}^p_\star\), of the best arm \(\star\), which in turn, is at least \(\mu_\star - \epsilon_p / 4\) according to the event \(\cE\). 

    Furthermore, according to the event \(\cE\), we have:
    \begin{equation}\label{eq:a2}
        \hat{\mu}^p_i \leq \mu_i + \epsilon_p / 4\,.
    \end{equation}
    
    By combining~\Cref{eq:a1} and~\Cref{eq:a2}, we obtain: 
    \begin{align*}
        \hat{\mu}^{p}_i - \hat{\mu}^{p}_{\max} + \epsilon_p \le \mu_{i} - \mu_{\star} + \epsilon_p / 2 + \epsilon_p = \frac{3}{2} \epsilon_p - \Delta_i < 0.
    \end{align*}
    The above inequality shows that \(\hat{\mu}^p_i < \hat{\mu}^p_{\max} -\epsilon_p\) holds due the large gap \(\Delta_i = \mu_{\star} - \mu_i > \frac{3}{2}\epsilon_{p}\). 
    Consequently, if \(i\) fails to meet the condition for inclusion in \(I_{p + 1}\), \(i\) will not present in \(I_{p + 1}\). This concludes the proof of the lemma. 
\end{proof}

As \(\epsilon_P \le \Delta_{[2]}/4\) and for any \(i \in I\) we have \(\Delta_i > \Delta_{[2]} \ge 4\epsilon_P > \frac{3}{2}\epsilon_P\), it follows that \(i\) satisfies the conditions of~\Cref{lem:eliminate-large-gap}. By combining~\Cref{lem:best} and~\Cref{lem:eliminate-large-gap}, we deduce that if event \(\mathcal{E}\) holds true, then \(I_{P + 1} = \{\star\}\), and \Cref{alg:main} outputs the correct arm.

We next bound the number of samples that \Cref{alg:main} makes.

\begin{lemma}\label{lem:bound-pull}
    Conditioning on the event \(\E\) defined in \Cref{eq:event-E} holds, the sample complexity of \Cref{alg:main} is at most 
    \begin{equation*}
        O\left(\log \left(\frac{nP}{\delta} \right) \sum_{i =2}^n \frac{n^{2/P}}{\Delta^2_{[i]}}\right)\,.
    \end{equation*}
\end{lemma}

\begin{proof}
    For any suboptimal arm $\arm_{i}$, define the value \(p(i) \triangleq \min\{p \ge 0 \mid \Delta_i > \frac{3}{2}\epsilon_p\}\). For an optimal arm, we define \(p(\star) \triangleq P\). We note that it is correctly defined because \(\frac{3}{2}\epsilon_P = \frac{3}{2} \frac{\Delta_{[2]}}{4} < \Delta_{[2]} \le \min_i \{\Delta_i\}\), which means \(\forall i : p(i) \le P\). We split the set of arms into two parts. The first set is the set of arms with small gaps
    \begin{align*}
        S \triangleq \{i \in I\mid \Delta_i \le 3 n \Delta / 2\},
    \end{align*}
    and the set of arms with big gaps
    \begin{align*}
        B \triangleq \{i \in I\mid \Delta_i > 3 n \Delta / 2\}\,.
    \end{align*}
    \noindent
    We define \(T_{S}\) and \(T_{B}\) as the number of arm pulls used by the sets \(S\) and \(B\), and \(T\triangleq T_{S}+T_{B}\) is the total number of arm pulls. We bound the two terms in what follows. 
    Furthermore, we first look at arms from \(S\) \emph{except} \(\armstar\). For \(\arm_{i}\) with gap parameter \(\Delta_{i}\), due to the definition of \(\epsilon_p\), we have that 
    \begin{align*}
        \frac{3}{2} n^{1/P} \epsilon_{p(i)}\ge \Delta_i > \frac{3}{2} \epsilon_{p(i)},
    \end{align*}
    and consequently,
    \begin{eqnarray}\label{eq:eps-lower}
        \epsilon_{p(i)} \ge \frac{2\Delta_i}{3 n^{1/P}}\,.
    \end{eqnarray}

    By \Cref{lem:eliminate-large-gap}, we have that the number of pulls for \(\arm_{i}\) is bounded by \(T_{p(i)}\). By the definition of \(T_{p}\) and~\Cref{eq:eps-lower}, we have

    \begin{align}\label{eq:basic-small}
        T_{p(i)} = \frac{8 \log(2 n (P + 1) / \delta)}{\epsilon^2_{p(i)} \log e} \le \frac{72 \ln (2 n (P + 1) / \delta) n^{2/P}}{4 \Delta^2_i \log e}\,.
    \end{align} 
    
    For the optimal arm, by \Cref{lem:best} and \Cref{lem:eliminate-large-gap}, the number of pulls assigned to the optimal arm is equal to
    \begin{align}\label{eq:basic-optimal}
        T_{P} = \frac{8 \ln (2 n (P + 1) / \delta)}{\epsilon^2_{P} \log e} = \frac{128 \ln (2 n (P + 1) / \delta)}{\Delta_{[2]}^2 \log e}.
    \end{align}

    All arms from the set \(B\) are eliminated after the first round, the number of pulls assigned in the first round is bounded by \(T_0 \le \frac{128 \ln (2 n (P + 1) / \delta)}{n^2\Delta_{[2]}^2}\). Thus, the total number of pulls assigned for arms from \(B\) is at most 
    \begin{align}\label{eq:basic-large}
        T_B \leq nT_0 \le \frac{128 \ln (2 n (P + 1) / \delta)}{n\Delta_{[2]}^2}.
    \end{align}

    Therefore, we have 
    \begin{align*}
    	T &= T_{S}+T_{B}\\
    	  &\leq T_{P} + \sum_{i\neq \star} T_{r(i)} + T_{B} \tag{by \Cref{lem:eliminate-large-gap}}\\
    	  &\leq \frac{128 \ln (2 n (P + 1) / \delta)}{\Delta_{[2]}^2 \log e} + \sum_{i\neq \star} \frac{72 \ln (2 n (P + 1) / \delta) n^{2/P}}{4 \Delta^2_i \log e} + T_{B} \tag{by \Cref{eq:basic-small} and \Cref{eq:basic-optimal}}\\
    	  &\leq \frac{128 \ln (2 n (P + 1) / \delta)}{\Delta_{[2]}^2 \log e} + \sum_{i\neq \star} \frac{72 \ln (2 n (P + 1) / \delta) n^{2/P}}{4 \Delta^2_i \log e} + \frac{128 \ln (2 n (P + 1) / \delta)}{n\Delta_{[2]}^2} \tag{by \Cref{eq:basic-large}}\\
    	  & = O\left(\log \left(\frac{nP}{\delta} \right) \sum_{i =2}^n \frac{n^{2/P}}{\Delta^2_{[i]}}\right),
    \end{align*}
    as desired.
\end{proof}

\paragraph{Finalizing the proof of \Cref{thm:basic}.} 
\Cref{lem:bound-E,lem:best,lem:eliminate-large-gap} guarantee that the \Cref{alg:main} outputs the best arm with probability at least \(1 - \delta\). The combination of \Cref{lem:bound-E} and \Cref{lem:bound-pull} gives the bound on the sample complexity. Finally, \Cref{lem:main-alg-memory} bounds the memory usage. 
\begin{observation}
\label{obs:delta-lower-bound}
We observe that \Cref{alg:main} can be extended to situations where the exact value of the optimality gap \(\Delta_{[2]}\) is unknown. Instead, we only have access to a lower bound \(\gamma\) for \(\Delta_{[2]}\) (we should use \(\epsilon_p = \gamma n^{1-\frac{p}{P}}\)). Under these circumstances, the analysis requires slight modifications. Although the correctness analysis remains unchanged, the bound of sample complexity needs modification. Specifically, when categorizing arms into sets \(B\) and \(S\), we should use the value \(\gamma n\) instead of \(\Delta_{[2]}n\). The contribution of arms from set \(S\) remains unchanged. However, the bound on the contributions of arms from set \(B\) to the total number of pulls is now bounded by \(O\left(\log\left(\frac{nP}{\delta}\right)\frac{1}{n\gamma^2}\right)\). Consequently, if we only have access to the lower bound \(\gamma\), the guarantees of \Cref{thm:basic} remain intact, except for a minor additive overhead of \(O\left(\log\left(\frac{nP}{\delta}\right)\frac{1}{n\gamma^2}\right)\) in the sample complexity.
\end{observation}

\begin{remark}
\label{rmk:logn-overhead}
Note that the sample complexity bound for \Cref{alg:main} with $P=O(\log(n))$ does \emph{not} have the dependency on the $\log\log(\frac{1}{\Delta_{[i]}})$ factor as in \cite{KarninKS13,JamiesonMNB14}. Instead, we have an extra $O(\log(n))$ multiplicative factor. We remark that the bound is \emph{not} trivial to obtain. For the worst-case optimal bounds, there exists a simple single-pass streaming algorithm that finds the best arm with $O(n\log(n)/\Delta^2_{[2]})$ sample complexity and a memory of a single arm. The algorithm is simply to keep the best arm on the fly with $O(\log(n)/\Delta^2_{[2]})$ samples on each arm. However, for the instance-sensitive sample complexity, it is unclear how to simulate this simple algorithm with only the knowledge of $\Delta_{[2]}$. Furthermore, since the parameters $\Delta_{[i]}$ can be arbitrarily small for any $i$, our multiplicative factor of $O(\log(n))$ can be better for infinitely many instances.
\end{remark}

\begin{remark}
    We observe that our algorithm is very similar to the elimination algorithms of \cite{JinH0X21,KarninKS13}, albeit our ``search'' of the gap parameter starts from $O(\sqrt{n}\Delta_{[2]})$. We rearmark that if we change the algorithm of \cite{JinH0X21,KarninKS13} by choosing the starting value for the gap as \(O(\sqrt{n}\Delta_{[2]})\), it appears that we can get a similar result with the number of passes \(P=\log(n)\) and sample complexity $\tilde{O}(\sum_{i=2}^{n}\frac{1}{\Delta^2_{[2]}})$. However, if we do not make the extra changes as in our algorithm, it is unclear how to obtain the trade-off between the sample complexity and the number of passes.
\end{remark}

\FloatBarrier
\section{Experiments}
\label{sec:experiment}
We present the empirical results in this section. For multi-pass streaming MABs algorithms, there are two objectives we want to optimize: the sample complexity and the pass efficiency. Our main experimental result is that compared to existing algorithms for streaming MABs, our algorithm exhibits significant advantages on both fronts.

\subsection{Experiment settings} 
We compare our algorithm with two benchmark algorithms: $i).$ the AW algorithm: the single-pass algorithm by \cite{AssadiW20}, which only uses a single pass over the stream, requires the knowledge of $\Delta_{[2]}$, and uses the worst-case optimal sample complexity of $\Theta(\frac{n}{\Delta^2_{[2]}})$; and $ii).$ the JHTX algorithm: the $O(\log(1/\Delta_{[2]}))$-pass algorithm by \cite{JinH0X21}, which does not require the knowledge of $\Delta_{[2]}$ and achieves the instance-sensitive near-instance optimal $O(\sum_{i=2}^{n}1/\Delta^2_{[i]}\cdot \log\log(1/\Delta_{[i]}))$ sample complexity, but has to use more passes than ours. We implemented the algorithm and track the number of arm pulls and passes. We pick the constants large enough so that the algorithms do \emph{not} fail in any of the runs.

We consider instances of $2000$ arms in $3$ settings, namely the \emph{uniform} setting, in which the mean rewards of the arms are from a uniform distribution supported on $[0,1]$; the \emph{arithmetic progression} setting, in which the mean rewards of the arms follow an arithmetic progression; and the \emph{cluster} setting, where there is one arm with higher mean reward ($0.9$), and all other arms are divided into two clusters with mean rewards $0.899$ and $0.898$, respectively. In each setting, we take $30$ independent runs, and we report the error bars to avoid statistical influx. 

\subsection{Experimental results}
Our main finding in the experiments is that our algorithm consistently outperforms both the algorithms of \cite{AssadiW20} and \cite{JinH0X21} in terms of sample complexity. Furthermore, compared to \cite{JinH0X21}, our algorithm uses significantly less passes. This confirms the theoretical analysis and demonstrates the strong practical value of our algorithm.
\paragraph{Results in the uniform setting.} The comparison between the sample complexity and the number of passes can be found in \Cref{fig:experiments-uniform}. Note that since the algorithm of \cite{AssadiW20} always uses a single pass, we do not report it in \Cref{fig:uniform-exp-pass}. From the figure, it can be found that the our algorithm has the best sample complexity among the algorithms. The sample complexity of the AW algorithm is considerably higher than the two others, which is understandable since in the instance with arithmetic progression means, we have $\frac{n}{\Delta^2_{[2]}} \gg \sum_{i=2}^{n}1/\Delta^2_{[i]}\cdot \log\log(1/\Delta_{[i]})$. comparing to the JHTX algorithm, our algorithm achieves lower mean sample complexity, and it is more stable. Finally, as we can see from \Cref{fig:uniform-exp-pass}, our algorithm uses a much smaller passes than JHTX.

\begin{figure}
	\centering
	\begin{subfigure}{0.5\textwidth}
		\centering
		\includegraphics[scale=0.3]{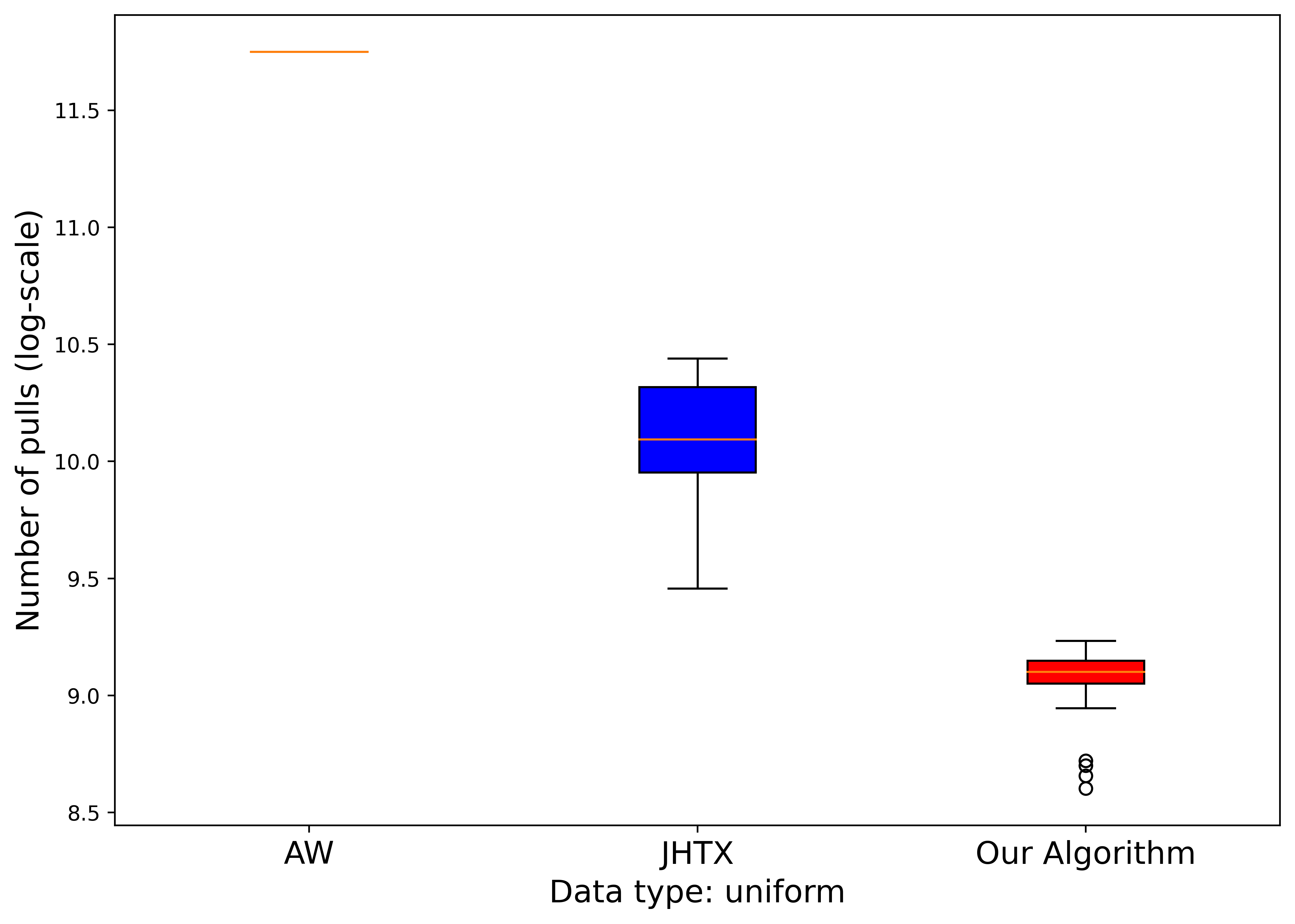}
		\caption{Comparison of samples between algorithms}
		\label{fig:uniform-exp-sample}
	\end{subfigure}%
	\begin{subfigure}{0.5\textwidth}
		\centering
		\includegraphics[scale=0.3]{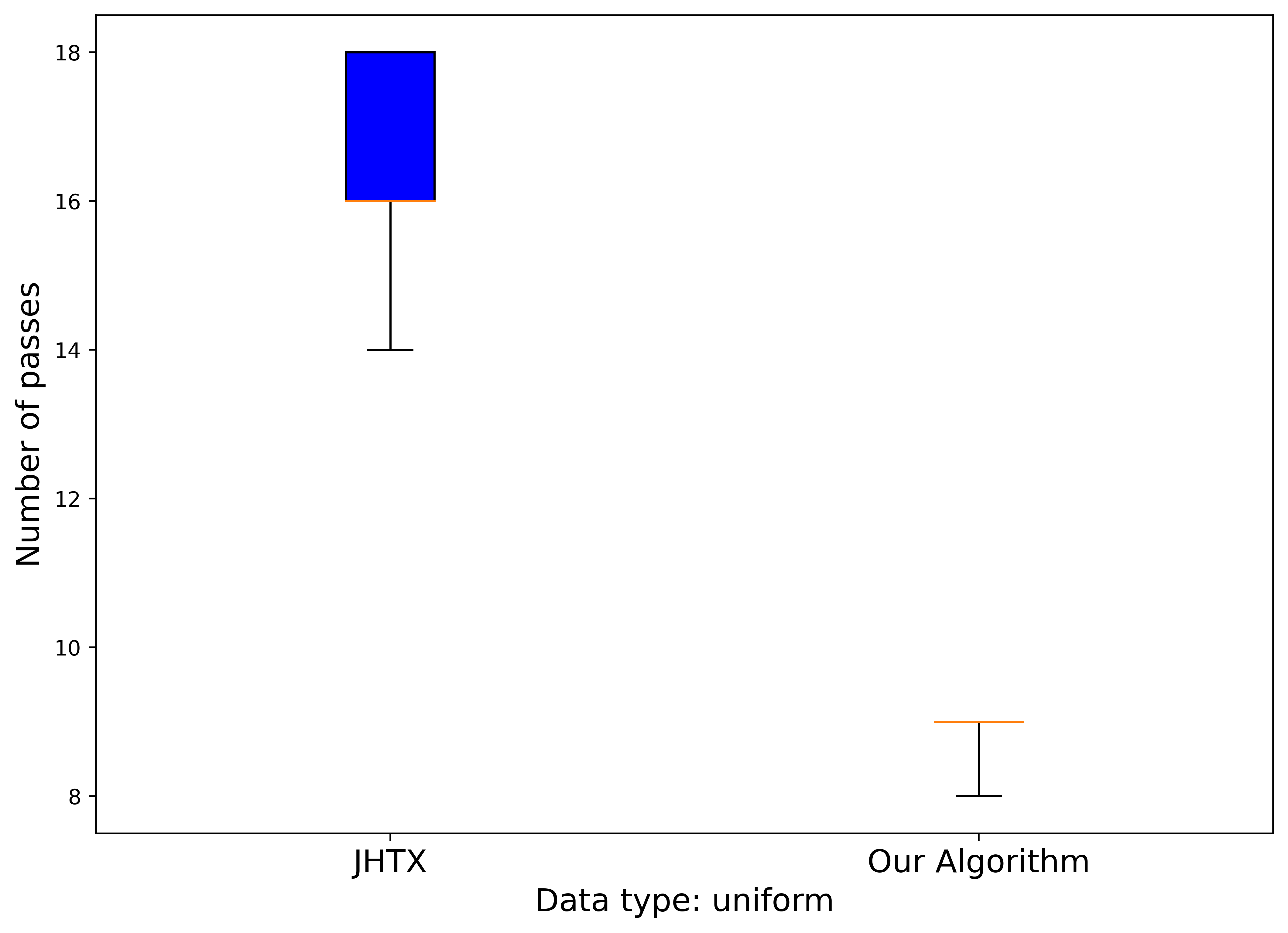}
		\caption{Comparison of passes between algorithms}
		\label{fig:uniform-exp-pass}
	\end{subfigure}
	\caption{The comparison between algorithms on the sample complexity and the number of passes in the \emph{uniform setting}. Samples numbers are taken $\log_{10}(\cdot)$ for better illustration. The graphs are reported by $30$ independent runs. AW stands for the single-pass algorithm of \cite{AssadiW20}, and JHTX stands for the single-pass algorithm of \cite{JinH0X21}.}
	\label{fig:experiments-uniform}
\end{figure}

To better illustrate the performance comparisons, we also summarize results in \Cref{tab:exp-uniform}. The numbers are rounded to two decimal places (and for the sample complexity, we do this after converting to the scientific notation). The table gives a clearer illustration of the better sample complexity of our algorithm: our algorithm in fact achieves a $10\times$ better sample efficiency than JHTX.
\begin{table}[!h]
	\centering
	\captionsetup{justification=centering}
	\caption{\label{tab:exp-uniform} The comparison between algorithms on the sample complexity and the number of passes in the \emph{uniform setting}.}
	\begin{tabular}{|l|l|l|}
		\hline
		& Mean samples & Mean passes \\ \hline
		AW & $5.62\times 10^{11}$ & -- \\ \hline
		JHTX & $1.41\times 10^{10}$ & 16.4 \\ \hline
		Our Algorithm & $1.18\times 10^{9}$ & 8.83  \\ \hline
	\end{tabular}
	
\end{table}

\paragraph{Results in the arithmetic progression setting.} The comparison between the sample complexity and the number of passes can be found in \Cref{fig:experiments-progression} and \Cref{tab:exp-progression}. This setting is very similar to the uniform setting, barring the fact that the rewards between arms are more stable and regularized. From the figure and the table, it can be observed that our algorithm again demonstrates the best sample complexity, and outperforms the second-best JHTX algorithm by almost $10\times$ in sample complexity. Again, since $\frac{n}{\Delta^2_{[2]}} \gg \sum_{i=2}^{n}1/\Delta^2_{[i]}\cdot \log\log(1/\Delta_{[i]})$, the AW algorithm uses a significantly larger number of arm pulls.  

\begin{figure}
	\centering
	\begin{subfigure}{0.5\textwidth}
		\centering
		\includegraphics[scale=0.3]{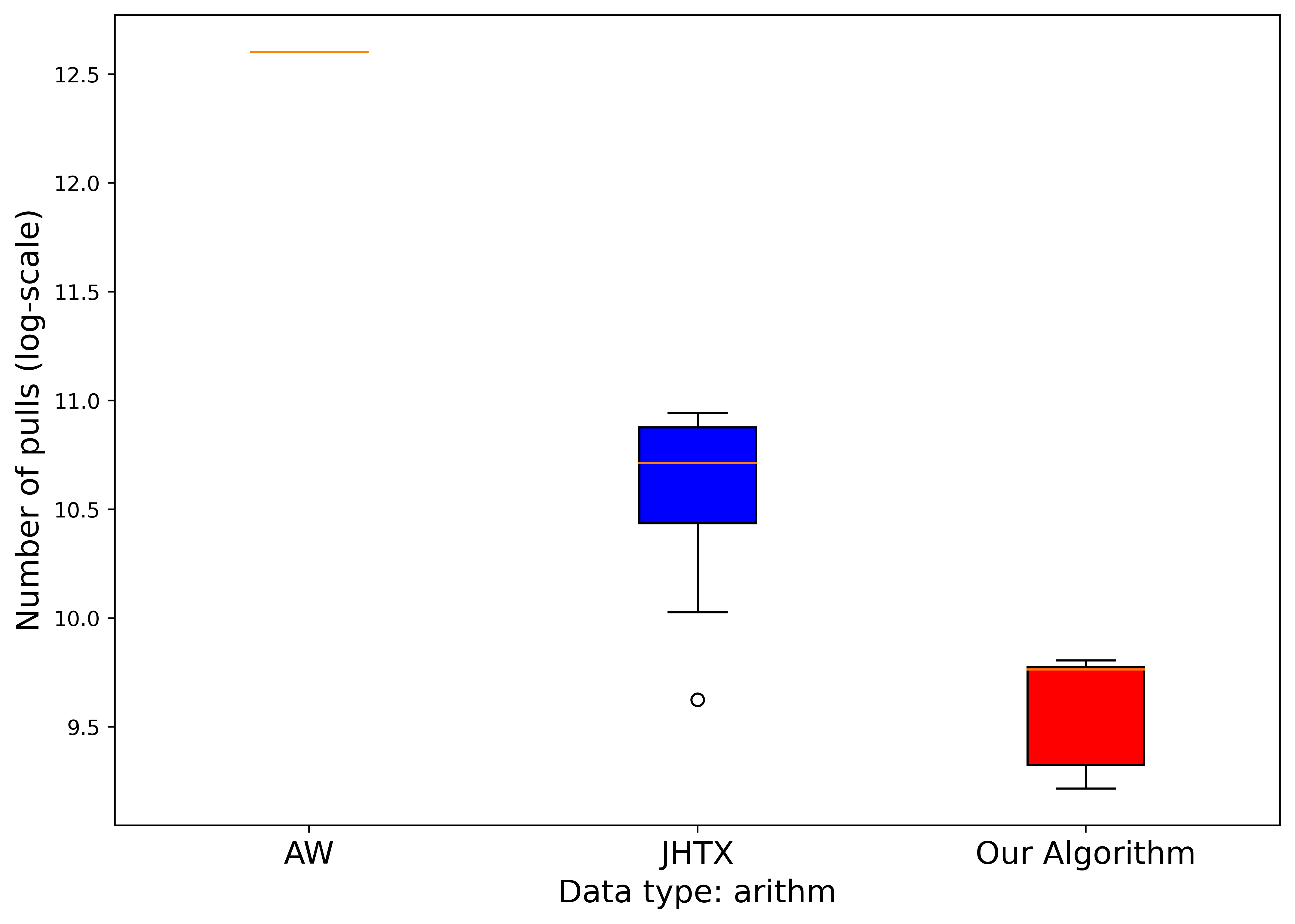}
		\caption{Comparison of samples between algorithms}
		\label{fig:progression-exp-sample}
	\end{subfigure}%
	\begin{subfigure}{0.5\textwidth}
		\centering
		\includegraphics[scale=0.3]{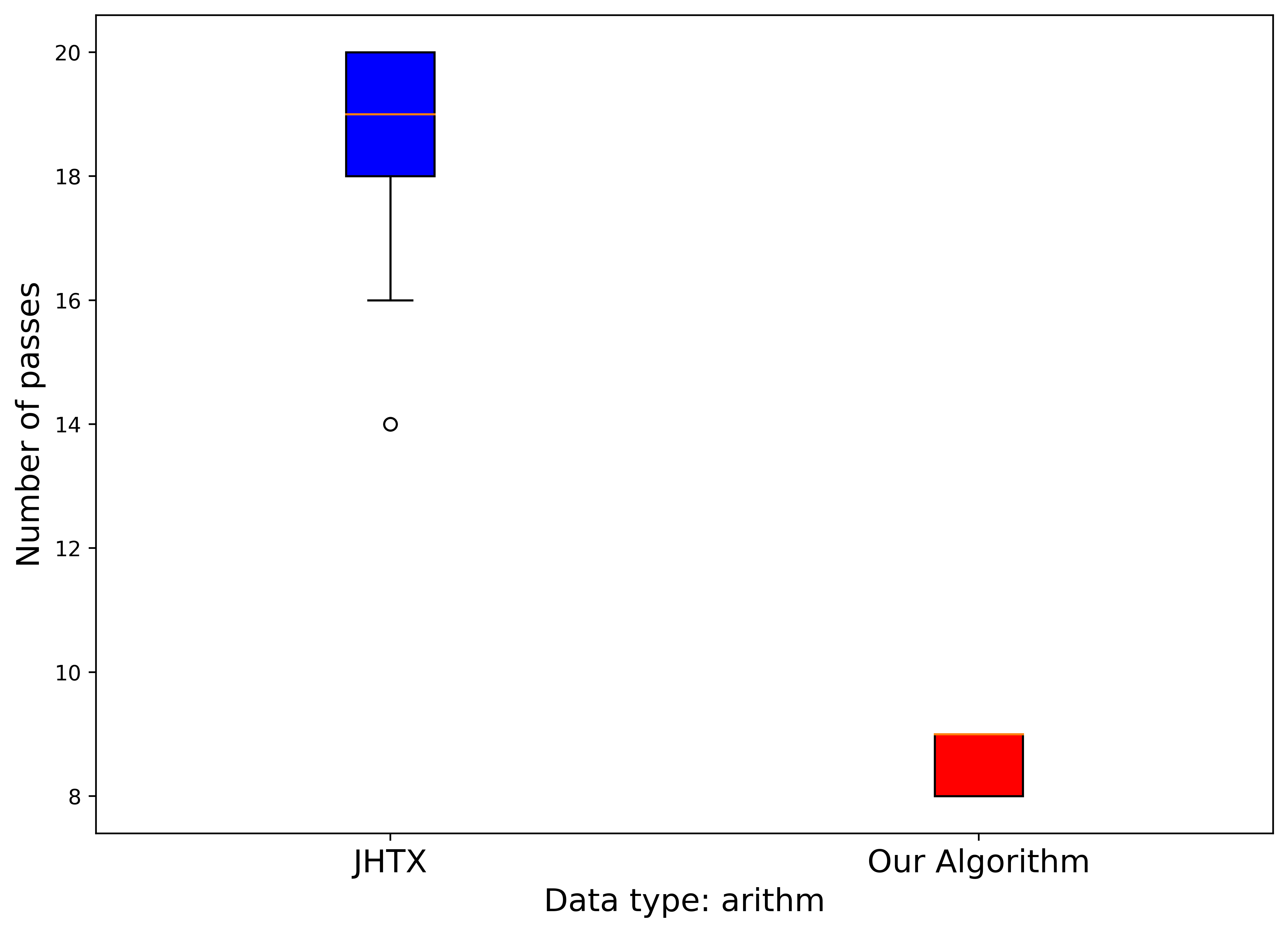}
		\caption{Comparison of passes between algorithms}
		\label{fig:progression-exp-pass}
	\end{subfigure}
	\caption{The comparison between algorithms on the sample complexity and the number of passes in the \emph{arithmetic progression setting}. Samples numbers are taken $\log_{10}(\cdot)$ for better illustration. The graphs are reported by $30$ independent runs. AW stands for the single-pass algorithm of \cite{AssadiW20}, and JHTX stands for the single-pass algorithm of \cite{JinH0X21}.}
	\label{fig:experiments-progression}
\end{figure}

\begin{table}[!h]
	\centering
	\captionsetup{justification=centering}
	\caption{\label{tab:exp-progression} The comparison between algorithms on the sample complexity and the number of passes in the \emph{arithmetic progression setting}.}
	\begin{tabular}{|l|l|l|}
		\hline
		& Mean samples & Mean passes \\ \hline
		AW & $4.01\times 10^{12}$ & -- \\ \hline
		JHTX & $5.05\times 10^{10}$ & 18.53 \\ \hline
		Our Algorithm & $4.61\times 10^{9}$ & 8.67  \\ \hline
	\end{tabular}
\end{table}

\paragraph{Results in the clustered instance setting.} We now come to the setting for clustered instances, and the results can be shown in \Cref{fig:experiments-cluster} and \Cref{tab:exp-cluster}. Note that in this case, we have that $\frac{n}{\Delta^2_{[2]}} \approx \sum_{i=2}^{n}1/\Delta^2_{[i]}\cdot \log\log(1/\Delta_{[i]})$, which is the reason the AW algorithm offers a very competitive sample complexity in \Cref{fig:cluster-exp-sample}. Interestingly, the JHTX algorithm is using many arm pulls in this setting. We suspect the reason is that regardless of the actually gaps, the gap-elimination procedure in the JHTX algorithm has to search from large gaps, and waste many sample and passes without eliminating any arm. We also note that our algorithm again demonstrate the best sample complexity in this setting. 

\begin{figure}
	\centering
	\begin{subfigure}{0.5\textwidth}
		\centering
		\includegraphics[scale=0.3]{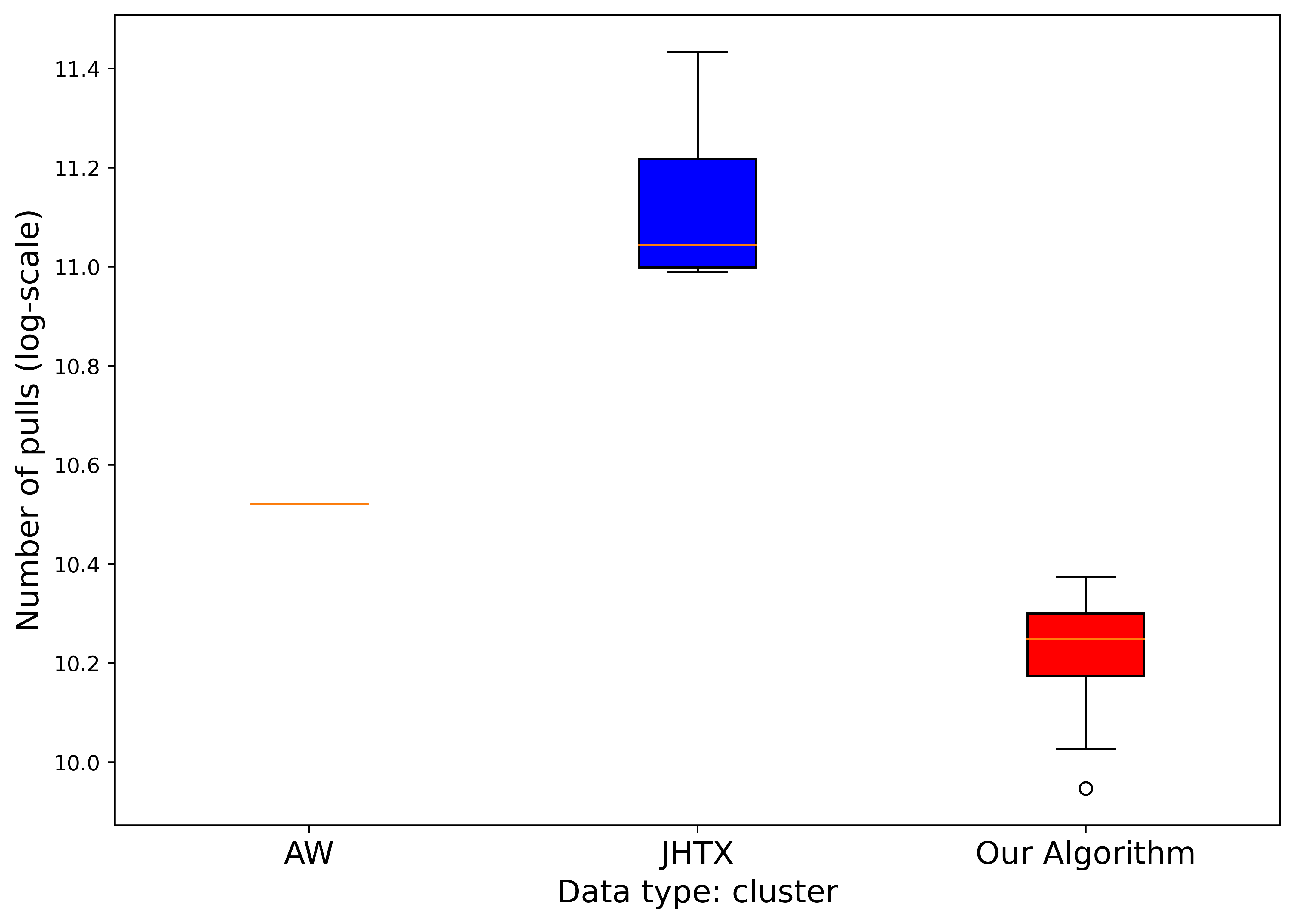}
		\caption{Comparison of samples between algorithms}
		\label{fig:cluster-exp-sample}
	\end{subfigure}%
	\begin{subfigure}{0.5\textwidth}
		\centering
		\includegraphics[scale=0.3]{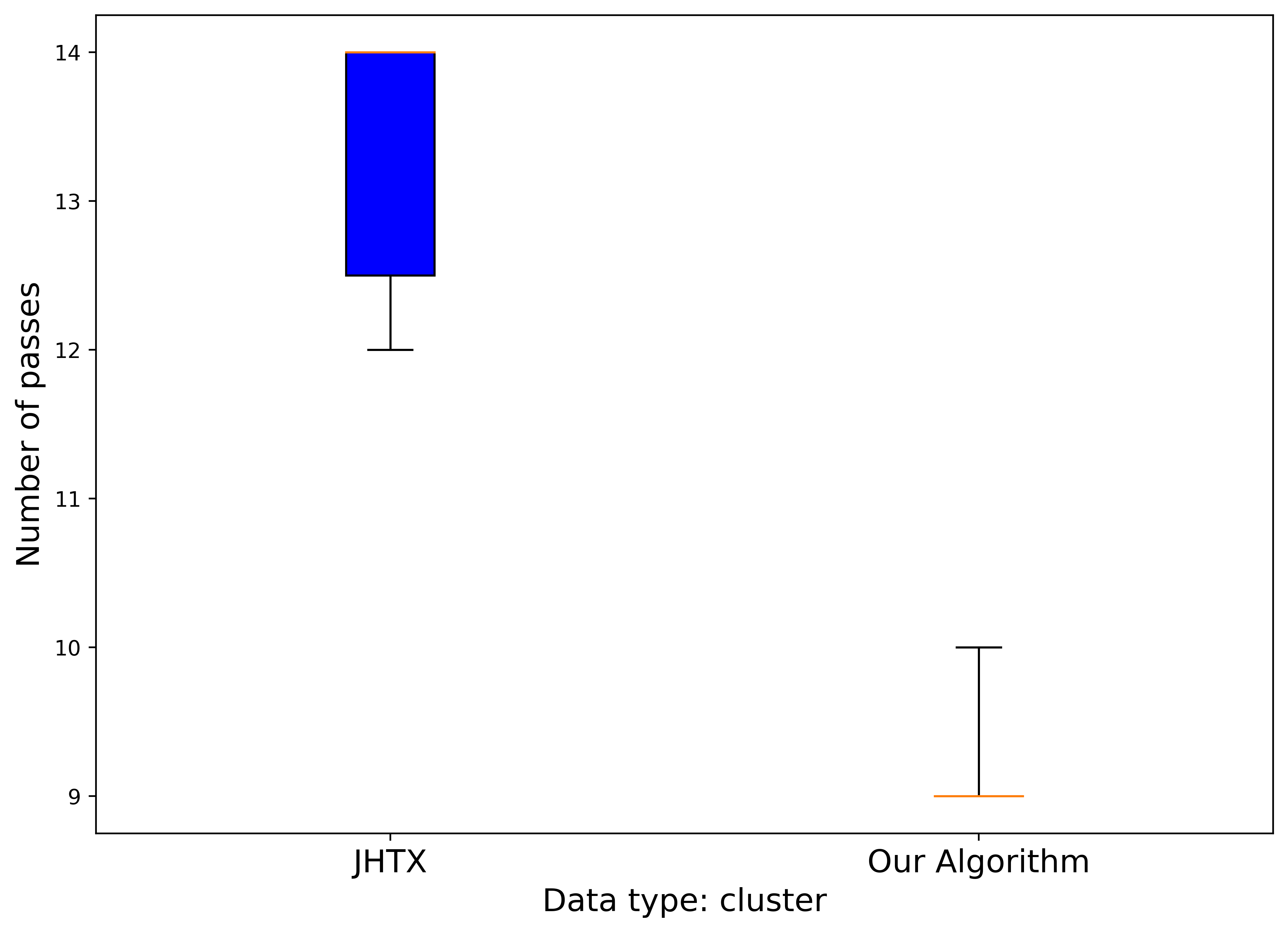}
		\caption{Comparison of passes between algorithms}
		\label{fig:cluster-exp-pass}
	\end{subfigure}
	\caption{The comparison between algorithms on the sample complexity and the number of passes in the \emph{arithmetic progression setting}. Samples numbers are taken $\log_{10}(\cdot)$ for better illustration. The graphs are reported by $30$ independent runs. AW stands for the single-pass algorithm of \cite{AssadiW20}, and JHTX stands for the single-pass algorithm of \cite{JinH0X21}.}
	\label{fig:experiments-cluster}
\end{figure}

\begin{table}[!h]
	\centering
	\captionsetup{justification=centering}
	\caption{\label{tab:exp-cluster} The comparison between algorithms on the sample complexity and the number of passes in the \emph{uniform setting}.}
	\begin{tabular}{|l|l|l|}
		\hline
		& Mean samples & Mean passes \\ \hline
		AW & $3.32\times 10^{10}$ & -- \\ \hline
		JHTX & $1.38\times 10^{11}$ & 13.47 \\ \hline
		Our Algorithm & $1.73\times 10^{10}$ & 9.03  \\ \hline
	\end{tabular}
\end{table}

\FloatBarrier

\bibliographystyle{plain}
\bibliography{paper}

\appendix

\section{Standard technical tools}
\label{sec:standatd-tech-tools}

\subsection{Concentration Inequalities}

We use the following variant of the well-known Chernoff-Hoeffding bound.

\begin{lemma}[Chernoff-Hoeffding Inequality]\label{lem:chernoff}
	Let \(X_1, \dotsc, X_n \in [0, 1]\) be independent random variables. Let \(X = \sum_{i = 1}^n X_i\). For any \(t \ge 0\), it holds that 
	\begin{align*}
		\Pr[X \ge \bE[X] + t] \le \exp\left(\frac{-2t^2}{n}\right)
	\end{align*}
	and
	\begin{align*}
		\Pr[X \le \bE[X] - t] \le \exp\left(\frac{-2t^2}{n}\right)\,.
	\end{align*}
\end{lemma}


\subsection{Statistical Distances and Properties}
\label{sub-app:stat-dist}
We frequently use the well-known total variation distance (TVD) and Kullback–Leibler divergence (KL divergence) in our proof. In this section, we provide their formal definition and properties.

\paragraph{Total variation distance.} 
We start with the definition of the total variation distance (TVD) between two distributions. 
\begin{definition}
	\label{def:tvd}
	Let $X$ and $Y$ be two random variables supported over the same $\Omega$, and let $\mu_X$ and $\mu_Y$ be their probability measures. The total variation distance (TVD) is between $X$ and $Y$ is defined as
	\begin{align*}
		\tvd{X}{Y} = \sup_{\Omega'\subseteq \Omega}\card{\mu_X(\Omega')-\mu_Y(\Omega')}.
	\end{align*}
	In particular, when the random variables are discrete, we have
	\begin{align*}
		\tvd{X}{Y} = \frac{1}{2}\sum_{\omega\in \Omega}\card{\mu_X(\omega)-\mu_Y(\omega)}.
	\end{align*}
\end{definition}

The total variation distance satisfied the symmetric property, i.e., $\tvd{X}{Y}=\tvd{Y}{X}$, and the triangle inequality, i.e., for three random vairables $X,Y,Z$, there is $\tvd{X}{Y}+\tvd{Y}{Z}\geq \tvd{X}{Z}$.

\paragraph{KL divergence.} We now introduce the Kullback–Leibler divergence (KL divergence) as an alternative statistical distance of TVD.

\begin{definition}[KL divergence]
	\label{def:kl-div}
	Let $X$ and $Y$ be two discrete random variables supported over the same $\Omega$, and let their distributions be $\mu_{X}$ and $\mu_{Y}$.The KL divergence between $X$ and $Y$, denoted as $\kl{X}{Y}$, is defined as 
	\begin{align*}
		\kl{X}{Y} = \sum_{\omega\in \Omega} \mu_{X}(\omega)\log\paren{\frac{\mu_{X}(\omega)}{\mu_{Y}(\omega)}}.
	\end{align*}
\end{definition}

Unlike the TVD, KL divergence is \emph{not} symmetric in general and does \emph{not} follow the triangle inequality.

\paragraph{The Properties of the TVD and KL divergence}

We shall use the following standard properties of KL-divergence and TVD defined. For the proof of this results, see the excellent textbook by Cover and Thomas~\cite{CoverT06}. 

We first connect the KL-divergence to TVD with the following celebrated Pinsker's inequality. 

\begin{fact}[Pinsker's inequality]
	\label{fact:pinsker}
	For any random variables $X$ and $Y$ supported over the same $\Omega$, 
	\begin{align*}
		\tvd{X}{Y} \leq \sqrt{\frac{1}{2}\cdot \kl{X}{Y}}.
	\end{align*}
\end{fact} 

We then present the key properties we used in our proof for the KL divergence, which includes the Chain rule and the conditional KL-divergence.
\begin{fact}[Chain rule of KL divergence]
	\label{fact:kl-chain-rule}
	For any random variables $X=(X_{1}, X_{2})$ and $Y=(Y_{1},Y_{2})$ be two random variables, 
	\begin{align*}
		\kl{X}{Y} = \kl{X_{1}}{Y_{1}} + \Exp_{x \sim X_1} \kl{X_{2}\mid X_{1}=x}{Y_{2}\mid Y_{1}=x}.
	\end{align*}
\end{fact}


\begin{fact}[Conditioning cannot decrease KL-divergence]\label{fact:kl-conditioning}
	For any random variables $X,Y,Z$, 
	\[
	\kl{X}{Y} \leq \Exp_{z \sim Z} \kl{X \mid Z=z}{Y \mid Z=z}. 
	\]
\end{fact}

In our proofs, we slighlty abuse the notation to let $\kl{X|Z}{Y|Z}$ denoting $\Exp_{z \sim Z} \kl{X \mid Z=z}{Y \mid Z=z}$.

The following fact characterizes the error of MLE for the source of a sample based on the TVD of the originating distributions. 

\begin{fact}
	\label{fact:distinguish-tvd}
	Suppose $\mu$ and $\nu$ are two distributions over the same support $\Omega$; then, given one sample $s$ from the following distribution
	\begin{itemize}
		\item With probability $\rho$, sample $s$ from $\mu$;
		\item With probability $1-\rho$, sample $s$ from $\nu$;
	\end{itemize}
	The best probability we can decide whether $s$ came from $\mu$ or $\nu$ 
	is 
	\[
	\max(\rho, 1-\rho) + \min(\rho, 1-\rho)\cdot\tvd{\mu}{\nu}.
	\]
\end{fact}

We frequently use the calculation of KL divergence between Bernoulli random variables. The following fact is a standard upper bound for KL divergence on Bernoulli random variables:
\begin{fact}
	\label{fct:bernoulli-KL}
	Let two random variables be distributed with $\bern{p}$ and $\bern{q}$, there is
	\begin{align*}
		\kl{\bern{p}}{\bern{q}}\leq \frac{(p-q)^2}{q\cdot (1-q)}.
	\end{align*}
\end{fact}

\Cref{fct:bernoulli-KL} implies the following the upper bound for KL-divergence, which we frequently use in our proof.

\begin{claim}
\label{clm:bernoulli-KL}
Let two random variables be distributed with $\bern{1/2+\alpha}$ and $\bern{1/2+\beta}$ such that $\max\{\alpha,\beta\}\leq \frac{1}{6}$, there are
\begin{align*}
	& \kl{\bern{1/2+\alpha}}{\bern{1/2+\beta}} \leq 8\cdot (\beta-\alpha)^2\\
	& \kl{\bern{1/2+\beta}}{\bern{1/2+\alpha}} \leq 8\cdot (\beta-\alpha)^2.
\end{align*}
\end{claim}
\begin{proof}
We prove the first inequality, since the second inequality follows the same calculation. The calculation is as follows.
\begin{align*}
	\kl{\bern{1/2+\alpha}}{\bern{1/2+\beta}} & \leq (\beta-\alpha)^2 \cdot \frac{1}{1/4-\beta^2}\\
	& \leq \frac{36}{8}\cdot (\beta-\alpha)^2 \tag{using $\beta\leq \frac{1}{6}$}\\
	&\leq 8\cdot (\beta-\alpha)^2.
\end{align*}
\end{proof}

\subsection{Information Theory Tools}
\label{subsec:info-theory}
We present the definition and basic properties of the information-theoretic tools in our proofs. For a random variable $X$, we let $\HH(X)$ be the \emph{Shannon entropy} of $X$, defined as follows
\begin{definition}[Shannon entropy]
	\label{def:entropy}
	Let $X$ be a discrete random variable with distributions $\mu_{X}$, the \emph{Shannon entropy} of $X$ is defined as
	\begin{align*}
		\HH(X) \triangleq \expect{\log(1/\mu(X))} =\sum_{x \in \text{supp}(X)} \mu(x)\cdot \log(\frac{1}{\mu(x)}),
	\end{align*}
	where $\text{supp}(X)$ is the support of $X$. If $X$ is a Bernoulli random variable, we use $H_2(p)$ to denote its Shannon entropy, where $P$ is the probability for $X=1$.
\end{definition}

We now give the definition of conditional entropy and mutual information.
\begin{definition}
	\label{def:mutual-info}
	Let $X$, $Y$ be two random variables, we define the \emph{conditional entropy} as \[\HH(X|Y)=\mathbb{E}_{y \sim Y}[\HH(X\mid Y=y)].\] 
	With conditional entropy, we can define the \emph{mutual information} between $X$ and $Y$ as \[\II\paren{X;Y}\triangleq \HH(X)-\HH(X\mid Y) = \HH(Y)-\HH(Y|X).\]
\end{definition}

In our proof, we use the following information-theoretic fact (see e.g.~\cite{CoverT06}) that connects mutual information with the KL-divergence.
\begin{fact}
	\label{fct:mutual-info-and-kl}
	Let $X$, $Y$ be  discrete random variables, there is
	\[\II(X;Y) = \Exp_{y\sim Y}\bracket{\kl{X\mid Y=y}{X}}.\]
\end{fact}

\section{A Streaming Algorithm with Memory Efficiency in Both Arms and Statistics}\label{sec:ub-stat-efficient}
In this section we describe algorithm that achieves memory efficiency for both number of arms and statistics. In particular, we show a $P$-pass algorithm with a single-arm memory and maintaining \(O(P)\) bits of statistics. The sample complexity becomes \(O\left(P \log \left(\frac{nP}{\delta}\right) \cdot \sum_{i = 2}^{n} \frac{n^{2/P}}{\Delta^2_{[i]}}\right)\) in this algorithm, which is still in the range of $\tilde{O}\paren{\sum_{i = 2}^{n} \frac{n^{2/P}}{\Delta^2_{[i]}}}$ by picking $P=O(\log(n))$. The formal statement is as follows.

\begin{theorem}
	\label{thm:improved}
	For any \(P \geq 1\), \Cref{alg:improved} is an algorithm that given a streaming MABs instance and the value of $\Delta_{[2]}$, finds the best arm with probability at least \(1-\delta\) using most
	\[O\left(P \log \left(\frac{nP}{\delta}\right) \cdot \sum_{i = 2}^{n} \frac{n^{2/P}}{\Delta^2_{[i]}}\right)\] 
	arm pulls, a memory of a single arm, and at most $O(P)$ bits of statistics.
\end{theorem}

We refer the readers to \Cref{subsec:model} for the formal definition of the memory complexity on the number of arms and the statistics. Compared to the algorithm in \Cref{sec:basic}, the idea to reduce the memory usage it to not store sets \(I_p\) for each \(p\). Instead, we simply store the maximum estimated mean for each pass, and simulate the results of the previous passes by resampling to get the new empirical mean. As a result, we pay a $P$ factor overhead for the arm pulls on each arm, but do not need to explicitly maintain the indices of the arms or query the transcript $\pi$. We present the description of the algorithm in \Cref{alg:improved}.

\FloatBarrier
\begin{algorithm}
	\caption{Stream-Elimination-Re}\label{alg:improved}
	\KwIn{Stream \(I\), parameter $P$, gap parameter \(\Delta_{[2]}\), and confidence parameter \(\delta\)}
	\KwOut{Best arm}
	Set \(n \gets \abs{I}\)\; 
	Maintain the index of the returned arm $\itilde \gets \perp$\;
	Let \(\epsilon_p \triangleq n^{1-i/P}\Delta_{[2]} / 4\) for \(p = 0, \dotsc, P\) \;
	Let \(T_p \triangleq \frac{8}{\epsilon^2_p \log e} \log\left(\frac{2 n (P + 1)^2}{\delta}\right) \) for \(p = 0, \dotsc, P\)\; 
	\For{\(p = 0, \dotsc, P\)}{
		Initialize \(\hat\mu^p_{\max} \gets -\infty\)\;
		Initialize the number of eliminated arms \(c \gets 0\)\;
		\ForEach{\(i \in I\) in the arrival order}{
			\For{\(j = 0, \dotsc, p\)}{
				Pull arm \(i\) until the number of pulls reach \(T_j\) times and compute the estimated mean \(\hat\mu^{pj}_i\)\;
				\label{alg:improved-eliminate}\If{\(\hat\mu^{pj}_i < \hat{\mu}^j_{\max} - \epsilon_j\)}{
					Update \(c \gets c + 1\) \;
					Eliminate arm \(i\) and exit loop \; 
				}
				\If{arm \(i\) is not  eliminated}{
					\If{\(\hat\mu^p_{\max} < \hat\mu^{pj}_i\)}{
						Update \(\hat\mu^p_{\max} \gets \hat\mu^{pj}_i\) and \(\itilde \gets i\)\;
					}
				}
			}
		}
		\If{\(c = n - 1\)}{
			\Return{The arm of index \(\itilde\)}\;
		}
	}
\end{algorithm}

\FloatBarrier

We start with observing the memory efficiency for \Cref{alg:improved} on both stored arms and bits of statistics. Formally, we show that
\begin{lemma}
\label{lem:improved-memory}
The maximum size of the memory of \Cref{alg:improved} is at most a single arm, and the maximum size of statistics \Cref{alg:improved} matains is $O(P)$.
\end{lemma}
\begin{proof}
	For the memory complexity of the number of arms, note that we can one-the-fly keep the arm with index $\itilde$, and we never need to store more than the $\arm_{\itilde}$. 
	
	For the size of the statistics, note that we do \emph{not} keep the empirical means for individual arms in the memory, and we only maintain the mean estimation of $\hat\mu^p_{\max}$ and at most $P$ different values of $\hat\mu^{p{j}}$. Therefore, the auxiliary number of bits to maintain is at most $O(P)$. Finally, for the size of $\pitilde$ (defined in \Cref{subsec:model}), we can always update the estimation of mean on-the-fly, and we never query a record of arm pull for more than once. Therefore, the size of $\pitilde$ is at most $1$. The desired memory bound is obtained by summarizing the above cases.
\end{proof}

For the correctness and the sample complexity, the proof strategy is similar to the proof for \Cref{alg:main}. 
We start with the definition of the event when all estimated means are close to the real values of means. Let's define an event \(\F\) as follows:

\begin{align}\label{eq:event-F}
	\F \triangleq \left\{\forall{i \in I}, p \in \{0, \dotsc, P\}, j \in \{0, \dotsc, p\} : \abs{\hat\mu^{pj}_i - \mu_i} \le \epsilon_r / 4\right\}
\end{align}
Where, \(\hat\mu^{pj}_i\) is the empirical mean of the samples drawn from distribution \(i\) in \(p\)-th pass after \(T_j\) pulls. 

We can now state the following lemma: 

\begin{lemma}\label{lem:bound-F}
	Let \(\F\) be the event defined in~\Cref{eq:event-F}. Then, \[\Pr\left[ \neg \F\right] \le \delta\,. \]
\end{lemma}
\begin{proof}
	By the union bound, we have:
	\begin{align}\label{eq:event-F-sum}
		\Pr\left[\neg \F\right] \le \sum_{i \in I} \sum_{p = 0}^{P} \sum_{j = 0}^{p} \Pr\left[\abs{\hat{\mu}^{pj}_i - \mu_i} > \epsilon_r/4\right]\,.
	\end{align}
	By using the Chernoff-Hoeffding inequality (\Cref{lem:chernoff}) we have: 
	\begin{align}\label{eq:event-F-term}
		\Pr\left[\abs{\hat\mu^{pj}_i-\mu_i} > \frac{\epsilon_r}{4}\right] \leq 2\exp\left(\frac{\epsilon^2_r T_r}{8}\right) \leq \frac{\delta}{n(P + 1)^2}\,.
	\end{align}
	
	Combining \Cref{eq:event-F-sum} and \Cref{eq:event-F-term}, we get: 
	\begin{align*}
		\Pr\left[\neg \F\right] \le n {(P + 1)}^2 \frac{\delta}{n(P + 1)^2} = \delta
	\end{align*}
\end{proof}

\begin{lemma}\label{lem:improved-best}
	 Conditioning on the event \(\F\) defined in \Cref{eq:event-F} holds, then for any \(p \in [P]\) we have that the arm \(\star\) is not eliminated in the \(p\)-th pass. 
\end{lemma}

\begin{proof}
	We need to prove that the if condition from \Cref{alg:improved-eliminate} in \Cref{alg:improved} does not hold for \(i = \star\). To prove this, we start by assuming the opposite. Assume that in \(p\)-th iteration for some \(j \in \{0, \dotsc, p\}\) we have 
	\begin{align*}
		\hat\mu^{pj}_{\star} < \hat\mu^j_{\max} - \epsilon_j
	\end{align*}
	and consequently
	\begin{align}\label{eq:improved-upper}
		\hat\mu^{pj}_{\star} < \hat\mu^{jj}_i - \epsilon_j
	\end{align}
	for some \(i \in I\). 
	
	However, from the definition of the event \(\F\) (\Cref{eq:event-F}) we have \(\hat\mu^{pj}_\star \geq \mu_\star - \epsilon_j / 4\) and \(\hat\mu^{jj}_i \leq \mu_i + \epsilon_r / 4\). Consequently, we get: 
	\begin{align}\label{eq:improved-lower}
		\hat\mu^{pj} - \hat\mu^{jj}_i \ge \mu_\star - \mu_i - \epsilon_j / 2 \ge -\epsilon_j / 2 \,.
	\end{align}
	
	\Cref{eq:improved-upper} and \Cref{eq:improved-lower} contradict to each other, leading to a contradiction. Therefore, it is not possible for the arm \(\star\) be eliminated at any moment of work of \Cref{alg:improved}. 
\end{proof}

We next prove that any suboptimal arm with a large gap will be eliminated in each pass after before specific number of pulls. 

\begin{lemma}\label{lem:improved-suboptimal-bound}    
	Conditioning on the event \(\F\) defined in \Cref{eq:event-F} holds, and assume for arm \(i\) and value \(j\), arm \(i\) satisfies \(\Delta_i > \frac{3}{2}\epsilon_j\), then for any pass \(p \in \{j, \dotsc, P\}\), arm \(i\) will be eliminated after at most \(T_j\) pulls.
\end{lemma}
\begin{proof}
	Consider any suboptimal arm \(i\) and a value \(j\) such that \(\Delta_i > \frac{3}{2} \epsilon_j\). We aim to prove that the arm \(i\) should be eliminated after we make at most \(T_j\) pulls for any pass \(p \geq j\). Consider \(p\)-th pass, if the arm \(i\) is eliminated before we make \(T_j\) pulls, then we are done. 
	Consider the opposite case when we make at least \(T_j\) pulls. By using \Cref{lem:improved-best} and the definition of the event \(\F\), we have the following inequality: 
	\begin{align}\label{eq:improved-best-lower}
		\hat\mu^{p}_{\max} \ge \hat\mu^{pj}_{\star} \ge \mu_{\star} - \epsilon_j / 4\,.
	\end{align}
	This inequality indicates that the maximum estimated mean \(\hat\mu^{j}_{\max}\) in \(p\)-th pass is at least as large as the estimated mean \(\hat\mu^{pj}_{\star}\) of the best arm \(\star\), which is at least \(\mu_\star - \epsilon_j / 4\) by the event \(\F\). 
	
	Furthermore, by the event \(\F\), we have:
	\begin{align}\label{eq:improved-sub-upper}
		\hat\mu^{pj}_{i} \le \mu_i + \epsilon_j / 4\,.
	\end{align}
	
	Combining \Cref{eq:improved-best-lower} and \Cref{eq:improved-sub-upper}, we obtain: 
	\begin{align*}
		\hat\mu^{pj}_i - \hat\mu^j_{\max} + \epsilon_j \leq \mu_j - \mu_\star + \epsilon_j / 2 + \epsilon_j = \frac{3}{2}\epsilon_j - \Delta_i < 0\,.
	\end{align*}
	
	The above inequality shows that \(\hat\mu^{pj}_i < \hat\mu^{j}_{\max} - \epsilon_j\) holds due the large value of the gap \(\Delta_i > \frac{3}{2} \epsilon_j\). 
\end{proof}

As \(\epsilon_P \le \frac{\Delta_{[2]}}{4}\) and for any suboptimal arm \(i\) we have \(\Delta_i \ge \Delta_{[2]} \ge 4\epsilon_P > \frac{3}{2}\epsilon_P\), it follows that any suboptimal arm \(i\) will be eliminated in \(P\)-th pass (\Cref{lem:improved-suboptimal-bound}) and optimal arm \(\star\) will not be eliminated by \Cref{lem:improved-best}. Thus, \Cref{alg:improved} outputs the correct arm if the event \(\F\) holds.

\begin{lemma}\label{lem:improved-sample-complexity}
	Conditioning on the event \(\F\) defined in \Cref{eq:event-F} holds, the number of arm pulls used by \Cref{alg:improved} is at most 
	\begin{align*}
		O\left(P \log\left( \frac{1}{\delta}\right) \cdot \sum_{i = 2}^n \frac{n^{2/P}}{\Delta^2_{[i]}}\right).
	\end{align*}
\end{lemma}

\begin{proof}
	We split the set of arms \(I\) into two parts, \(B\) be the set arms with big gaps 
	\begin{align*}
		B \triangleq \left\{i \in I \mid \Delta_i > \frac{3n\Delta_{[2]}}{2}\right\}
	\end{align*}
	and small gaps
	\begin{align*}
		S \triangleq \left\{ i \in I \mid \Delta_i \le \frac{3n\Delta_{[2]}}{2}\right\}\,.
	\end{align*}
	
	Similar to the proof of \Cref{lem:bound-pull}, we define $T_{B}$ and $T_{S}$ to be the number of arm pulls used by the arms in $B$ and $S$, and $T\triangleq T_B + T_S$ is the total number of arm pulls. For any sub arm \(i\), we define the value \(p(i) \triangleq \min\{p \ge 0 \mid \Delta_i > \frac{3}{2}\epsilon_p\}\). For the optimal arm \(\star\), we have \(p(\star) = P\) by \Cref{lem:improved-suboptimal-bound}. We note that it is correctly defined because 
	\[\frac{3}{2}\epsilon_P = \frac{3}{2}\frac{\Delta_{[2]}}{4} < \Delta_{[2]}\,,\]
	and so on \(\forall{i \in I} : p(i) \le P\). 
	
	For arms from \(S\) due the definitions of \(\epsilon_r\) and \(p(i)\), we have that 
	\begin{align*}
		\frac{3}{2} n^{1/P} \epsilon_{p(i)} \ge \Delta_i > \frac{3}{2} \epsilon_{p(i)},
	\end{align*}
	and consequently, 
	\begin{align}\label{eq:improved-eps-lower}
		\epsilon_{p(i)} \ge \frac{2\Delta_i}{3 n^{1/P}}\,.
	\end{align}
	
	By \Cref{lem:improved-suboptimal-bound}, we have that the number of pulls for any suboptimal arm \(i \in S\) is bounded by the number of passes \(P + 1\) times \(T_{p(i)}\).  Consequently, by \Cref{eq:improved-eps-lower}, the number of pulls for arm \(i\) from \(S \setminus \{\star\}\)  the number of pulls is bounded
	\begin{align}\label{eq:improved-bound-small-gap}
		P T_{p(i)} \le P \frac{8}{\epsilon^2_{p(i)} \log{e}} \log\left(\frac{2 n (P + 1)^2}{\delta}\right) / \log(e) \le \frac{18 P }{\Delta^2_i \log{e} } \log\left(\frac{2 n (P + 1)^2}{\delta}\right)\,.
	\end{align}
	
	For the optimal arm the number of pulls is trivially bounded by \begin{align}\label{eq:improved-bound-optimal}
		P T_{P} \le P \frac{128}{\Delta_{[2]}^2 \log e} \log\left(\frac{2n{(P + 1)}^2}{\delta}\right)\,.
	\end{align}
	
	For arms from the set \(B\) by \Cref{lem:improved-suboptimal-bound} we have that the number of pulls assigned to an arm \(i \in B\) is bounded by 
	\begin{align*}
		P T_0 \le \frac{128}{n^2\Delta_{[2]}^2 \log e} \log \left(\frac{2n{(P + 1)}^2}{\delta}\right)\,.
	\end{align*}
	We note that the size of \(B\) is bounded by \(n\). Therefore, the total sample complexity for arms from \(B\) is bounded by 
	\begin{align}\label{eq:improved-bound-large-gap}
		n \cdot \frac{128}{n^2\Delta_{[2]}^2 \log e} \log \left(\frac{2n{(P + 1)}^2}{\delta}\right) \le \frac{128}{\Delta_{[2]}^2 \log e} \log \left(\frac{2n{(P + 1)}^2}{\delta}\right)\,.
	\end{align}
	
	As such, combining~\Cref{eq:improved-bound-small-gap}, \Cref{eq:improved-bound-optimal}, and \Cref{eq:improved-bound-large-gap}, we have
	\begin{align*}
		T &= T_{S} + T_{B}\\
		&\leq P \frac{128}{\Delta_{[2]}^2 \log e} \log\left(\frac{2n{(P + 1)}^2}{\delta}\right) + \sum_{i: i\neq \star} P \frac{18}{\Delta_{i}^2 \log e} \log\left(\frac{2n{(P + 1)}^2}{\delta}\right) + T_{B}\tag{by \Cref{eq:improved-bound-small-gap}, \Cref{eq:improved-bound-optimal}}\\
		&\leq O\left(P \log\left( \frac{nP}{\delta}\right) \cdot \sum_{i = 2}^n \frac{n^{2/P}}{\Delta^2_{[i]}}\right) + \frac{128}{\Delta_{[2]}^2 \log e} \log \left(\frac{2n{(P + 1)}^2}{\delta}\right)  \tag{by \Cref{eq:improved-bound-large-gap}}\\
		&= O\left(P \log\left( \frac{nP}{\delta}\right) \cdot \sum_{i = 2}^n \frac{n^{2/P}}{\Delta^2_{[i]}}\right),
	\end{align*} 
	as desired.
\end{proof}

\paragraph{Finalizing the proof of \Cref{thm:improved}.} The algorithm makes $P+1$ passes over the stream by the algorithm design. The memory efficiency is guaranteed by \Cref{lem:improved-memory}, and the correctness and the sample complexity are gueranteed by \Cref{lem:bound-F,lem:improved-best,lem:improved-suboptimal-bound,lem:improved-sample-complexity}. Combining the above completes the proof of \Cref{thm:improved}.


\begin{remark}
	\label{rmk:sample-lb-higher-polylog}
	By setting $P=O(\log(n))$, we get a sample complexity bound of $O(\log^2 (n)\cdot \sum_{i=2}^{n}\frac{1}{\Delta^{2}_{[i]}})$. We remark that our lower bound in \Cref{thm:lb-main} can also be extended to the sample complexity of $O(\log^2 (n)\cdot \sum_{i=2}^{n}\frac{1}{\Delta^{2}_{[i]}})$. In fact, by using $B=\Theta(\log(n)/\log\log(n))$, we can already extend the lower bound to $O((\frac{\log(n)}{\log\log(n)})^2\cdot \sum_{i=2}^{n}\frac{1}{\Delta^{2}_{[i]}})$ sample complexity. We can actually strengthen the bound on \Cref{prop:multi-pass-lb} to allow larger exponent on $B$ by using higher gaps between $\chi_{b}$ and $\chi_{b+1}$ (e.g., $\chi_{b+1}= ({1}/{6 C \log(n)})^{100} \cdot \chi_{b}$). To make the lower bound work, we will need to use smaller constant for $P\leq B$, e.g., at most ${1}/{10000}\cdot \log(n)/\log\log(n)$ passes, which is still in the range of $\Omega(\log(n)/\log\log(n))$.
\end{remark}

\end{document}